\documentclass{article}

\PassOptionsToPackage{numbers}{natbib}
\usepackage[final]{neurips_2020}

\title{Tree! I am no Tree! I am a Low Dimensional Hyperbolic Embedding}

\usepackage{amsmath,amsfonts,bm}

\def\eqref#1{equation~\ref{#1}}

\def\1{\bm{1}}

\DeclareMathAlphabet{\mathsfit}{\encodingdefault}{\sfdefault}{m}{sl}
\SetMathAlphabet{\mathsfit}{bold}{\encodingdefault}{\sfdefault}{bx}{n}

\newcommand{\R}{\mathbb{R}}

\DeclareMathOperator*{\argmin}{arg\,min}

\usepackage[utf8]{inputenc} 
\usepackage[T1]{fontenc}    
\usepackage{hyperref}       
\usepackage{booktabs}       
\usepackage{amsfonts}       
\usepackage{nicefrac}       
\usepackage{microtype}      
\usepackage{enumerate}
\usepackage{fancyhdr}
\usepackage{mathrsfs}
\usepackage{graphicx}
\usepackage{enumitem}
\usepackage{chngcntr}
\usepackage{bbm}
\usepackage{subfigure}
\usepackage[noend]{algpseudocode}
\usepackage{enumitem}
\usepackage{algorithm,algorithmicx}
\usepackage{amsmath,xparse}
\usepackage{amssymb}
\usepackage{amsthm}
\usepackage[table,svgnames]{xcolor}

\newcommand{\cellhim}[1]{\cellcolor{RoyalBlue!#1}}

\newcommand{\cellbest}{\cellcolor{Green!40}}

\newtheorem{thm}{Theorem}
\newtheorem{lem}{Lemma}

\newtheorem{defn}{Definition}
\newtheorem{prop}{Proposition}
\newtheorem{rem}{Remark}
\newtheorem{problem}{Problem}

\newenvironment{reflemma}[1]{\medskip\parindent 0pt{\bf Lemma \ref{#1}.}\em }{\vspace{1em}}
\newenvironment{reftheorem}[1]{\medskip\parindent 0pt{\bf Theorem \ref{#1}.}\em }{\vspace{1em}}

\newenvironment{refprop}[1]{\medskip\parindent 0pt{\bf Proposition \ref{#1}.}\em }{\vspace{1em}}

\numberwithin{equation}{section}

\author{%
   Rishi Sonthalia\thanks{Corresponding Author}\\
  Department of Mathematics\\
  University of Michigan\\
  Ann Arbor, MI, 48104 \\
  \texttt{rsonthal@umich.edu} \\
   \And
  Anna C. Gilbert \\
  Department of Statistics and Data Science \\
  Yale University \\
  New Haven, CT, 06510 \\
   \texttt{anna.gilbert@yale.edu} \\
}

\begin{document}

\maketitle

\begin{abstract}

Given data, finding a faithful low-dimensional hyperbolic embedding of the data is a key method by which we can extract hierarchical information or learn representative geometric features of the data. In this paper, we explore a new method for learning hyperbolic representations by taking a metric-first approach. Rather than determining the low-dimensional hyperbolic embedding directly, we learn a tree structure on the data. This tree structure can then be used directly to extract hierarchical information, embedded into a hyperbolic manifold using Sarkar's construction \cite{sarkar}, or used as a tree approximation of the original metric. To this end, we present a novel fast algorithm \textsc{TreeRep} such that, given a $\delta$-hyperbolic metric (for any $\delta \geq 0$), the algorithm learns a tree structure that approximates the original metric. In the case when $\delta = 0$, we show analytically that \textsc{TreeRep} exactly recovers the original tree structure. We show empirically that \textsc{TreeRep} is not only many orders of magnitude faster than previously known algorithms, but also produces metrics with lower average distortion and higher mean average precision than most previous algorithms for learning hyperbolic embeddings, extracting hierarchical information, and approximating metrics via tree metrics.
\end{abstract}

\section{Introduction}

Extracting hierarchical information from data is a key step in understanding and analyzing the structure of the data in a wide range of areas from the analysis of single cell genomic data \cite{singlecell}, to linguistics \cite{dhingra2018embedding}, computer vision \cite{khrulkov2019hyperbolic} and social network analysis \cite{suri}. In single cell genomics, for example, researchers want to understand the developmental trajectory of cellular differentiation. To do so, they seek techniques to visualize, to cluster, and to infer temporal properties of the developmental trajectory of individual cells.

One way to capture the hierarchical structure is to represent the data as a tree. Even simple trees, however, cannot be faithfully represented in low dimensional Euclidean space \cite{Linial1995}. As a result, a variety of remarkably effective hyperbolic representation learning methods, including~\citet{fb1,fb2,albert}, have been developed. These methods learn an embedding of the data points in hyperbolic space by first solving a non-convex optimization problem and then extracting the hyperbolic metric that corresponds to the distances between the embedded points. These methods are successful because of the inherent connections between hyperbolic spaces and trees. They do not, however, come with rigorous geometric guarantees about the quality of the solution. Also, they are slow.

In this paper, we present a metric first approach to extracting hierarchical information and learning hyperbolic representations. The important connection between hyperbolic spaces and trees suggests that the correct approach to learning hyperbolic representations is the metric first approach. That is, first, learn a tree that essentially preserves the distances amongst the data points and then embed this tree into hyperbolic space.\footnote{A similar idea is mentioned in~\citet{albert} for graph inputs rather than general metrics. They do not, however undertake a detailed exploration of the idea.} 
More generally, the metric first approach to metric representation learning is to build or to learn an appropriate metric first by constructing a discrete, combinatorial object that corresponds to the distances and then extracting its low dimensional representation rather than the other way around.

The quality of a hyperbolic representation is judged by the quality of the metric obtained. That is, we say that we have a good quality representation if the hyperbolic metric extracted from the hyperbolic representation is, in some way, faithful to the original metric on the data points. We note that finding a tree metric that approximates a metric is an important problem in its own right. Frequently, we would like to solve metric problems such as transportation, communication, and clustering on data sets. However, solving these problems with general metrics can be computationally challenging and we would like to approximate these metrics by simpler, tree metrics. This approach of approximating metrics via simpler metrics has been extensively studied before. Examples include dimensionality reduction \cite{MR737400} and approximating metrics by simple graph metrics \cite{bartal,spanner}. 

To this end, in this paper, we demonstrate that methods that learn a tree structure first outperform methods that learn hyperbolic embeddings directly. Additionally, we have developed a novel, extremely fast algorithm \textsc{TreeRep} that takes as input a $\delta$-hyperbolic metric and learns a tree structure that approximates the original metric. \textsc{TreeRep} is a new method that makes use of geometric insights obtained from the input metric to infer the structure of the tree. To demonstrate the effectiveness of our method, we compare \textsc{TreeRep} against previous methods such as \citet{constructTree} and \citet{nj} that also recover tree structures given a metric. There is also significant literature on approximating graphs via (spanning) trees with low stretch or distortion, where the algorithms take as input graphs, \emph{not metrics}, and output trees that are subgraphs of the original. We also compare against such algorithms \cite{alon,levelTrees,elkin,prim}. We show that when we are given only a metric and not a graph, then even if we use a nearest neighbor graph or treat the metric as a complete graph \textsc{TreeRep} is not only faster, but produces better results than \cite{constructTree,alon,bartal,levelTrees,prim} and comparable results to \cite{nj}. 

For learning hyperbolic representations, we demonstrate that \textsc{TreeRep} is over 10,000 times faster than the optimization methods from \citet{fb1,fb2}, and \citet{albert} while producing better quality results in most cases. This extreme decrease in time, with no loss in quality, is exciting as it allows us to extract hierarchical information from much larger data sets in single-cell sequencing, linguistics, and social network analysis - data sets for which such analysis was previously unfeasible.  
 
The rest of the paper is organized as follows. Section \ref{sec:background} contains the relevant background information. Section \ref{sec:treerep} presents the geometric insights and the \textsc{TreeRep} algorithm. In Section \ref{sec:exp}, we compare \textsc{TreeRep} against the methods from \citet{constructTree, alon, levelTrees,prim} and \citet{nj} in approximating metrics via tree metrics and against methods from \citet{fb1,fb2} and \citet{albert} for learning low dimensional hyperbolic embedding. We show that the methods that learn a good tree to approximate the metric, in general, find better hyperbolic representations than those that embed into the hyperbolic manifold directly.

\section{Preliminaries}
\label{sec:background}

The formal problem that our algorithm will solve is as follows\footnote{Note that the input to our problem are metrics and not graphs. Thus, we handle more general inputs as compared to \citet{alon,elkin,levelTrees}, and \citet{prim}.}.
\begin{problem} \label{problem}
    Given a metric $d$ find a tree structure $T$ such that the shortest path metric on $T$ approximates $d$.
\end{problem}

\begin{defn} \label{def:APSPmetric}  Given a weighted graph $G=(V,E,W)$ the shortest path metric $d_G$ on $V$ is defined as follows: $\forall u,v \in V$, $d_G(u,v)$ is the length of the shortest path from $u$ to $v$. \end{defn}

\textbf{$\delta$-Hyperbolic Metrics.}
Gromov introduced the notion of $\delta$-hyperbolic metrics as a generalization of the type of metric obtained from negatively curved manifolds \cite{Gromov1987}.

\begin{defn} Given a space $(X,d)$, the Gromov product of $x,y \in X$ with respect to a base point $w \in X$ is 
\[
	(x,y)_w := \frac{1}{2}\left(d(w,x) + d(w,y) - d(x,y) \right). 
\]
The Gromov product is a measure of how close $w$ is to the geodesic $g(x,y)$ connecting $x$ and $y$. 
\end{defn}

\begin{defn} A metric $d$ on a space $X$ is a $\delta$-hyperbolic metric on $X$ (for $\delta \ge 0$), if for every $w,x,y,z \in X$ we have that
\begin{equation} \label{eq:hyp}
	(x,y)_w \ge \min\big( (x,z)_w, (y,z)_w \big) - \delta. 
\end{equation}
In most cases we care about the smallest $\delta$ for which $d$ is $\delta$-hyperbolic. 
\end{defn}

An example of a $\delta$-hyperbolic space is the hyperbolic manifold $\mathbb{H}^k$ with $\displaystyle \delta = \tanh^{-1}\left(1/\sqrt{2}\right)$ \cite{23064}. 

\begin{defn} The hyperboloid model $\mathbb{H}^k$ of the hyperbolic manifold is $ \displaystyle \mathbb{H}^k = \{ x \in \R^{k+1} : x_0 > 0, x_0^2 - \sum_{i=1}^k x_i^2 = 1\}. $
\end{defn}

An important case of hyperbolic metrics is when $\delta = 0$. One important property of such metrics is that they come tree spaces. 

\begin{defn} \label{def:distree} A metric $d$ is a tree metric if there exists a weighted tree $T$ such that the shortest path metric $d_T$ on $T$ is equal to $d$.\footnote{Note, such metrics may have representations as graphs that are not trees, Section \ref{sec:treerep} has a simple example.} \end{defn}

\begin{defn} \label{def:metricgraph} Given a discrete graph $G = (V,E,W)$ the metric graph $(X,d)$ is the space obtained by letting $X = E \times [0,1]$ such that for any $(e,t_1), (e,t_2) \in X$ we have that 
$d((e,t_1),(e,t_2)) = W(e)\cdot |t_1-t_2|$. 
This space is called a tree space if $G$ is a tree. Here $E \times \{0,1\}$ are the nodes of $G$.
\end{defn}

\begin{defn} Given a metric space, $(X,d)$, two points $x,y \in X$, and a continuous function $f:[0,1] \to X$, such that $f(0) = x$, $f(1) = y$, and there is a $\lambda$ such that $d(f(t_1), f(t_2)) = \lambda|t_1-t_2|$, the geodesic $g(x,y)$ connecting $x$ and $y$ is the set $f([0,1])$.
\end{defn}

\begin{defn} \label{def:rtree} A metric space $T$ is a tree space (or a $\mathbb{R}$-tree) if any pair of its points can be connected with a unique geodesic segment, and if the union of any two geodesic segments $g(x,y),g(y,z) \subset T$ having the only endpoint $y$ in common, is the geodesic segment $g(x, z) \subset T$. \end{defn}

There are multiple definitions of a tree space. However, they are all connected via their metrics. \citet{tree0} tells us that a metric space is $0$-hyperbolic if and only if it is an $\mathbb{R}$-tree or a tree space. This result lets us immediately conclude that Definitions \ref{def:metricgraph} and \ref{def:rtree} are equivalent. Similarly, Definition \ref{def:APSPmetric} implies that Definition \ref{def:distree} and \ref{def:metricgraph} are equivalent. Hence all three definitions of tree spaces are equivalent. We note that trees are $0$-hyperbolic, and that $\delta = \infty$ corresponds to an arbitrary metric. Thus, $\delta$ is a heuristic measure for how close a metric is to a tree metric. 

\textbf{Trees as Hyperbolic Representation.}
The problem that is looked at by \cite{albert, fb1, fb2} is the problem of learning hyperbolic embeddings. That is, given a metric $d$, learn an embedding $X$ in some hyperbolic space $\mathbb{H}^k$. We, however, are proposing that if we want to learn a hyperbolic embedding, then we should instead learn a tree. In many cases, we can think of this tree as the hyperbolic representation. However, if we do want coordinates, this can be done as well. 

\citet{albert} give an algorithm that is a modification of the algorithm in \citet{sarkar} that can, in linear time, embed any weighted tree into $\mathbb{H}^k$ with arbitrarily low distortion (if scaling of the input metric is allowed). The analysis in \citet{albert} quantifies the trade-offs amongst the dimension $k$, the desired distortion, the scaling factor and the number of bits required to represent the distances in $\mathbb{H}^k$. We use these results to consider trees as hyperbolic representations. One possible drawback of this approach is that we may need a large number of bits of precision. Recent work, however, such as \citet{NIPS2019_8476} provides a solution to this issue. 
\section{Tree Representation}
\label{sec:treerep}

\begin{figure*}[!htb]
\centering
\subfigure[$\hat{T}$ when $\pi x = z$ and $(z,x)_w = (z,y)_w < d(w,z)$.]{\label{fig:a}\includegraphics[width=0.23\textwidth]{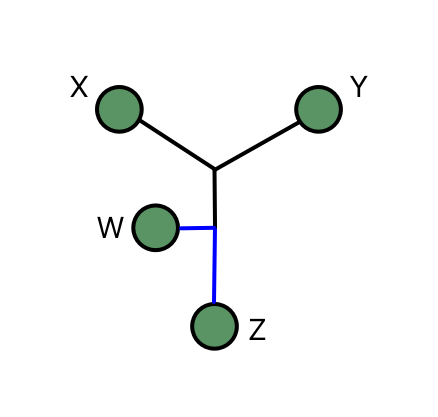}}\hfill
\subfigure[$\hat{T}$ when $\pi x = z$ and $(z,x)_w = (z,y)_w = d(w,z)$.]{\label{fig:e}\includegraphics[width=0.23\textwidth]{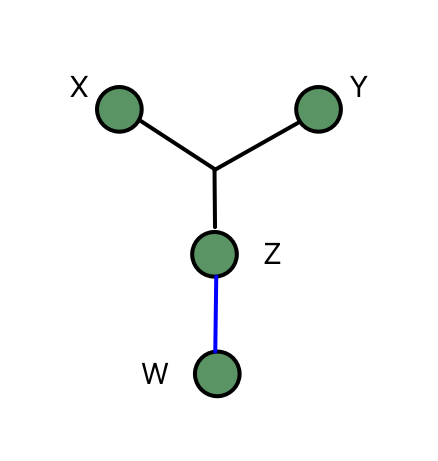}}\hfill
\subfigure[$\hat{T}$ when  $(y,x)_w = (y,z)_w = (x,z)_w \neq 0$]{\label{fig:h}\includegraphics[width=0.23\textwidth]{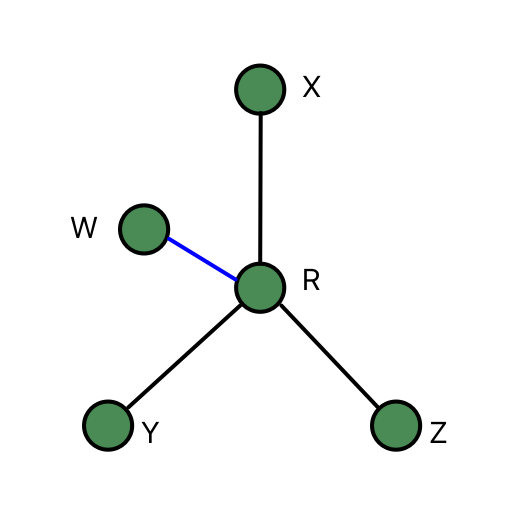}}\hfill
\subfigure[Universal Tree on  $x,y,z$.]{\label{fig:basicTree}\includegraphics[width=0.23\textwidth, height = 3.5cm]{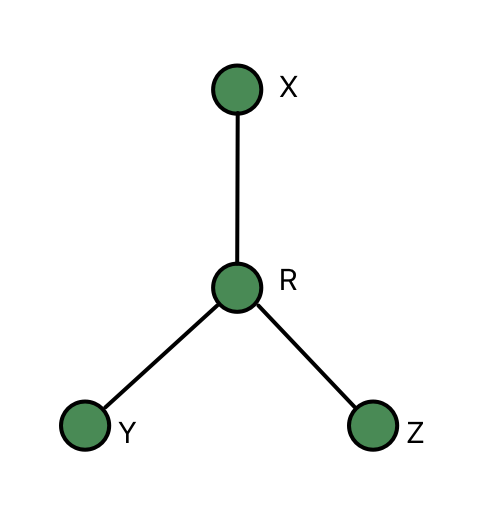}}\hfill
\caption{Figures showing the tree $\hat{T}$ from Lemma \ref{lem:structure} for $Zone_2(z)$ (a),  $Zone_1(z)$ (b), $Zone_1(r)$ (c), and the Universal tree (d).}
\label{fig:3trees}
\end{figure*}

To solve Problem \ref{problem} we present an algorithm \textsc{TreeRep} such that Theorem \ref{thm:metriconstruct} holds. 

\begin{thm} \label{thm:metriconstruct} 
Given $(X,d)$, a $\delta$-hyperbolic metric space, and $n$ points $x_1, \hdots, x_n \in X$, \textsc{TreeRep} returns a tree $(T,d_T)$. In the case that $\delta = 0$, $d_T = d$, and $T$ has the fewest possible nodes. \textsc{TreeRep} has a worst case run time $O(n^2)$. Furthermore the algorithm is embarrassingly parallelizable.  
\end{thm}

\begin{rem} In practice, we see that the run time for \textsc{TreeRep} is much faster than $O(n^2)$. 
\end{rem}

To better understand the geometric insights used to develop \textsc{TreeRep}, we first focus on the problem of reconstructing the tree structure from a tree metric. Algorithm \ref{alg:tree_struct2-main} and \ref{alg:tree_struct_recurse-main} present a high level version of the pseudo-code. The complete pseudo-code for \textsc{TreeRep} is presented in Appendix \ref{sec:code}.

\textsc{TreeRep} is a recursive, divide and conquer algorithm. The first main idea (Lemma \ref{lem:universal}) is that for any metric $d$ on three points, we can construct a tree $T$ with four nodes such that $d_T = d$. We will call such trees \emph{universal trees}. The second main idea (Lemma~\ref{lem:structure}) is that adding a fourth point to a universal tree can be done consistently and, more importantly, the additional point falls into one of seven different zones. Thus, \textsc{TreeRep} will first create a universal tree $T$ and then will sort the remaining data points into the seven different zones. We will then do recursion with each of the zones. 

\begin{lem} \label{lem:universal} 
Given a metric $d$ on three points $x,y,z$, there exists a (weighted) tree $(T,d_T)$ on four nodes $x,y,z,r$, such that $r$ is adjacent to $x,y,z$, the edge weights are given by $d_T(x,r) = (y,z)_x$, $d_T(y,r) = (x,z)_y$ and $d_T(z,r) = (x,y)_z$, and the metric $d_T$ on the tree agrees with $d$. 
\end{lem}

\begin{defn} \label{def:universal} 
The tree constructed in Lemma \ref{lem:universal} is the universal tree on the three points $x,y,z$. The additional node $r$ is known as a Steiner node.  \end{defn}   

An example of the universal tree can be seen in Figure \ref{fig:basicTree}. To understand the distinction between the seven different zones, we need to reinterpret Equation \ref{eq:hyp}. We know that for any tree metric, and any four points $w,x,y,z$, we have that 
\[ 
    (x,y)_w \ge \min\big( (x,z)_w, (y,z)_w\big). 
\]
This inequality implies that the smaller two of the three numbers $(x,y)_w, (x,z)_w$, and $(y,z)_w$ are equal. In this case, knowing which of the quantities are equal tells us the structure of the tree. Specifically, here $x,y,z$ will be the three points in our universal tree $T$ and $w$ will be the point that we want to sort. Then initially, we have four possibilities. The first possibility is that all three Gromov products are equal. This case will define its own zone. If this is not the case, then we have three possibilities depending on which two out of the three Gromov products are equal. Suppose we have that $
(x,y)_w = (x,z)_w$, then due to the triangle inequality, we have that $d(w,x) \ge (x,y)_w$. Thus, we will further subdivide this case into two more cases, depending on whether  $d(w,x) = (x,y)_w$ or $d(w,x) > (x,y)_w$. Each of these cases will define their own zone.  Examples of the different cases can be seen in Figure \ref{fig:3trees}. We can also see that there are two different types of zones. The first type is when we connect the new node directly to an existing node as seen in Figures \ref{fig:e} and \ref{fig:h}. The second type is when we connect $w$ to an edge as seen in Figure \ref{fig:a}. The formal definitions for the zones can be seen in Definition \ref{defn:zones}.


\begin{lem}\label{lem:structure} Let $(X,d)$ be a tree space. Let $w,x,y,z$ be four points in $X$ and let $(T,d_T)$ be the universal tree on $x,y,z$ with node $r$ as the Steiner node. Then we can extend $(T,d_T)$ to $(\hat{T},d_{\hat T})$ to include $w$ such that $d_{\hat T} = d$.
\end{lem}

\begin{defn} \label{defn:zones} Given a data set $V$ (consisting of data points, along with the distances amongst the points), a universal tree $T$ on $x,y,z \in V$ (with $r$ as the Steiner node), let us define the following two zone types. 
\begin{enumerate}[nosep]
    \item The definition for zones of type one is split into the following two cases. \begin{enumerate}
    \item $Zone_1(r) = \{ w \in V : (x,y)_w = (y,z)_w = (z,x)_w \}$
    \item For a given permutation $\pi$ on $\{x,y,z\}$, $Zone_1(\pi x) = \{ w \in V : (\pi x, \pi y)_w = (\pi x, \pi z)_w = d(w, \pi x) \}$
    \end{enumerate}
    \item For a given permutation $\pi$ on $\{x,y,z\}$, $Zone_2(\pi x) = \{ w \in V : (\pi x, \pi y)_w = (\pi x, \pi z)_w < d(w, \pi x)\}$
\end{enumerate}
\end{defn}

Using this terminology and our structural lemmas, we can describe a recursive algorithm that reconstructs the tree structure from a $0$-hyperbolic metric. Given a data set $V$ we pick three random points $x,y,z$ and construct the universal tree $T$. Then for all other $w \in V$, sort the $w$'s into their respective zones. Then for each of the seven zones we can recursively build new universal trees. For zones of type 1, pick any two points, $w_{i_1}, w_{i_2}$ and form the universal tree for $\pi x$ (or $r$),$w_{i_1}, w_{i_2}$. If there is only one node in this zone, connect it to $\pi x$ (or $r$). For zones of type 2, pick any one point, $w_{i_1}$ and form the universal tree for $\pi x,w_{i_1}, r$. Note that during the recursive step for zones of type 2, we create universal trees with Steiner nodes $r$ as one of the nodes. Hence we need to compute the distance from $r$ to all other nodes sent to that zone. We can calculate this when we first place $r$. Concretely, if $r$ is the Steiner node for the universal tree $T$ on $x,y,z$, then for any $w$, we will have that $d(w,r) = \max((x,y)_z,(y,z)_x,(z,x)_y)$. The proof for the consistency of this formula is in the proof of Lemma \ref{lem:structure}. 

Finally, to complete the analysis, the following lemma proves that we only need to check consistency of the metric within each zone to ensure global consistency. 

\begin{lem} \label{lem:consistency} Given $(X,d)$ a metric tree, and a universal tree $T$ on $x,y,z$, we have the following 
\begin{enumerate}[nosep]
\item \label{lem:zone1} If $w \in Zone_1(x)$, then for all $\hat{w} \not\in Zone_1(x)$, we have that $x \in g(w,\hat{w})$. 
\item \label{lem:zone2} If $w \in Zone_2(x)$, then for all $\hat{w} \not\in Zone_i(x)$ for $i=1,2$, we have that $r \in g(w,\hat{w})$. 
\end{enumerate} \end{lem}

\textbf{TreeRep for General $\delta$-Hyperbolic Metrics.}
\label{sec:treerepapprox}
Having seen the main geometric ideas behind \textsc{TreeRep}, we want to extend the algorithm to return an approximating tree for any given metric. 
For an arbitrary $\delta$-hyperbolic metric, Lemma \ref{lem:structure} does not hold. We can, however, modify it and leverage the intuition behind the original proof. 
Given four points $w,x,y,z$, we do not satisfy one of the conditions of Lemma \ref{lem:structure}, if all three Gromov products $(x,y)_w,(x,z)_w,(y,z)_w$ have distinct values. Nevertheless, we can still compute the maximum of these three quantities. Furthermore, since we have a $\delta$-hyperbolic metric, the smaller two products will be within $\delta$ of each other. Let us suppose that $(x,y)_w$ is the biggest. Then we place $w$ in $Zone_1(x)$ if and only if $d(z,w) = (y,z)_w$ \emph{or} $d(z,w) = (x,z)_w$.  Otherwise we place $w \in Zone_2(x)$. Note that when we have tree metric, we have that $d(z,w) = (y,z)_w$ if and only if $d(z,w) = (x,z)_w$. 

As shown by Proposition \ref{prop:dist}, when we do this, we are introducing a distortion of at most $\delta$ between $w$ and $y,z$. This suggests that when we do zone 2 recursive steps, we should pick the node that closest to $r$ as the third node for the universal tree. We see experimentally that this significantly improves the quality of the tree returned. Note, we do not have a global distortion bound for when the input is a general $\delta$-hyperbolic metric. However, as we will see experimentally, we tend to produce trees with low distortion. 

\begin{prop} \label{prop:dist} Given a $\delta$-hyperbolic metric $d$, the universal tree $T$ on $x,y,z$ and a fourth point $w$, when sorting $w$ into its zone ($zone_i(\pi x)$), \textsc{TreeRep} introduces an additive distortion of at most $\delta$ between $w$ and $\pi y, \pi z$. \end{prop}

\begin{algorithm}[!ht]
\caption{Metric to tree structure algorithm.}
\label{alg:tree_struct2-main}
	\begin{algorithmic}[1]
	\Function {\textsc{Tree structure}}{X, $d$}
		\State $T = (V,E,d') = \emptyset $ \quad
		\State Pick any three data points uniformly at random $x,y,z \in X$. 
		\State $T$ = \textsc{recursive\_step}($T,X,x,y,z,d,d_T,$)
		\State \textbf{return} $T$
	\EndFunction 
	\Function {\textsc{recursive\_step}}{$T,X,x,y,z,d,d_T,$}
		\State Construct universal tree for $x,y,z$ and sort the other nodes into the seven zones.
	\State Recurse for each of the seven zones by calling \textsc{Zone1\_Recursion} and \textsc{Zone2\_recursion}.
	\Return $T$
	\EndFunction
	\end{algorithmic}
\end{algorithm}

\begin{algorithm}[!htb]
\caption{Recursive parts of TreeRep.}
\label{alg:tree_struct_recurse-main}
	\begin{algorithmic}[1]
	\Function {\textsc{zone1\_recursion}}{$T$, $d_T$, $d$, $L$, $v$}
		\If{Length($L$) == 0}
			\Return$T$
		\EndIf
		\If{Length($L$) == 1}
			\State Let $u$ be the one element in $L$ and add edge $(u, v)$ to $E$ with weight $d_T(u,v) = d(u,v)$ 
			\State \textbf{return} $T$
		\EndIf
		\State Pick any two $u, z$ from $L$ and remove them from $L$
		\State \textbf{return} \textsc{recursive\_step}($T,L,v,u,z,d$, $d_T$)
	\EndFunction 
	
	\Function {\textsc{zone2\_recursion}}{$T$, $d_T$, $d$, $L$, $u$, $v$}
		\If{Length($L$) == 0}
			\Return $T$
		\EndIf
		\State Set $z$ to be the closest node to $v$ and delete edge $(u,v)$
		\State \textbf{return:} \textsc{recursive\_step}($T,L,v,u,z,d$, $d_T$)
	\EndFunction
	\end{algorithmic}
\end{algorithm}

\textbf{Steiner nodes.} A Steiner node is any node that did not exist in the original graph that one adds to it. We give a simple example to illustrate that Steiner nodes are necessary for reconstructing the correct tree. Additionally, we demonstrate that forming a graph and then computing any spanning tree (as done in \cite{alon,elkin,prim}) will not recover the tree structure. Consider 3 points $x,y,z$ such that all pairwise distances are equal to $2$. Then, the associated graph is a triangle and any spanning tree is a path. Then, the distance between the endpoints of the spanning tree is not correct; it has been distorted or stretched. The ``correct'' tree is obtained by adding a new node $r$ and connecting $x,y,z$ to $r$, and making all the edge weights equal to $1$. Thus, we need Steiner nodes when reconstructing the tree structure. Methods such as MST and LS \cite{alon} that do not add Steiner nodes will not produce the correct tree, when given a 0-hyperbolic metric, even though such algorithms do come with upper bounds on the distortion of the distances. In this setting, we want to obtain a tree that as accurately as possible represents the metric even at the cost of additional nodes; we do not simply want a tree that is a subgraph of a given graph. 


\section{Experiments}
\label{sec:exp}

In this section, we demonstrate the effectiveness of \textsc{TreeRep}. Additional details about the experiments and algorithms can be found in Appendix \ref{sec:expappendix}.\footnote{All code can be found at the following link \url{https://github.com/rsonthal/TreeRep}} 

For the first task of approximating metrics with tree metrics, we compare \textsc{TreeRep} against algorithms that find approximating trees; Minimum Spanning Trees (MST) \cite{prim},  \textsc{LevelTrees} (LT) \cite{levelTrees}, \textsc{Neighbor Join} (NJ) \cite{nj}, Low Stretch Trees (LS) \cite{alon,elkin}, \textsc{ConstructTree} (CT) \cite{constructTree}, and \textsc{ProbTree} (BT) \cite{bartal}. When comparing against such methods, we show that not only is \textsc{TreeRep} much faster than all of the above algorithms (except MST, and LS), but that \textsc{TreeRep} produces better quality metrics than MST, LS, LT, BT, and CT and metrics that are competitive with NJ. In addition to these methods, other methods such as UPGMA \cite{upgma} also learn tree structures. However, these algorithms have other assumptions on the data. In particular, for UPGMA, the additional assumption is that the metric is an ultrametric. Hence we do not compare against such methods. 

One important distinction between methods such as LS, MST, and LT and the rest, is that LS, MST, and LT require a graph as the input. This graph is crucial for these methods and hence sets these methods apart from the rest, as the rest only require a metric. 

For the second task of learning hyperbolic embeddings, we compare \textsc{TreeRep} against Poincare Maps (PM) \cite{fb1}, Lorentz Maps (LM) \cite{fb2}, PT \cite{albert}, and hMDS \cite{albert}. Since we can embed trees into $\mathbb{H}^k$ with arbitrarily low distortion, we think of trees as hyperbolic representations. When comparing against such methods, we show that \textsc{TreeRep} is not only four to five orders of magnitude faster, but for low dimensions, and in many high dimensional cases, produces better quality embeddings. 

We first perform a benchmark test for tree reconstruction from tree metrics. Then, for both tasks, we test the algorithms on three different types of data sets. First, we create synthetic data sets by sampling random points from $\mathbb{H}^{k}$. Second, we will take real world biological data sets that are believed to have hierarchical structure. Third, we consider metrics that come from real world unweighted graphs. In each case, we will show that \textsc{TreeRep} is an extremely fast algorithm that produces as good or better quality metrics. We will evaluate the methods on the basis of computational time, and the average distortion, as well as mean average precision (MAP) of the learned metrics.\footnote{MAP is used in \cite{fb1, fb2, albert}, while average distortion is used in \cite{albert}. The definitions are in the appendix. }

\begin{rem}
\textsc{TreeRep} is a randomized algorithm, so all numbers reported are averaged over 20 runs. The best number produced by \textsc{TreeRep} can be found in the Appendix.
\end{rem}

\textbf{Tree Reconstruction Experiments.}
\label{sec:treecon}
Before experimenting with general $\delta$-hyperbolic metrics, we benchmark our method on $0$-hyperbolic metrics. To do this, we generate random synthetic $0$-hyperbolic metrics. More details can be found in Appendix \ref{sec:expappendix}. 
Since \textsc{TreeRep} and \textsc{NJ} are the only algorithms that are theoretically guaranteed to return a tree that is consistent with the original metric, we will run this experiment with these two algorithms only. We compare the two algorithms based on their running times and the number of nodes in the trees. As we can see from Table \ref{table:recon}, \textsc{TreeRep} is a much more viable algorithm at large scales. Additionally, the trees returned by \textsc{NJ} have double the number of nodes as the original trees. Contrarily, the trees returned by \textsc{TreeRep} have exactly the same number of nodes as the original trees. 

\begin{table}[!ht]
\caption{Time taken by Nj and TreeRep to reconstruct the tree structure.}
\centering
\begin{tabular}{cccccccc}
\toprule
n & 11 & 40 & 89 & 191 & 362 & 817 & 1611 \\ \midrule
TR & \cellbest 0.053 &  0.23 &\cellbest 0.0017 &\cellbest 0.0039 &\cellbest 0.02 &\cellbest 0.08 &\cellbest 0.12 \\
NJ & 0.084 & \cellbest 0.0016 & 0.0067 & 0.036 & 0.18 & 1.7 & 15 \\
\bottomrule
\end{tabular}
\label{table:recon}
\end{table}

\begin{figure}
\centering
\subfigure[Varied Dimension.]{\includegraphics[width=0.49\linewidth]{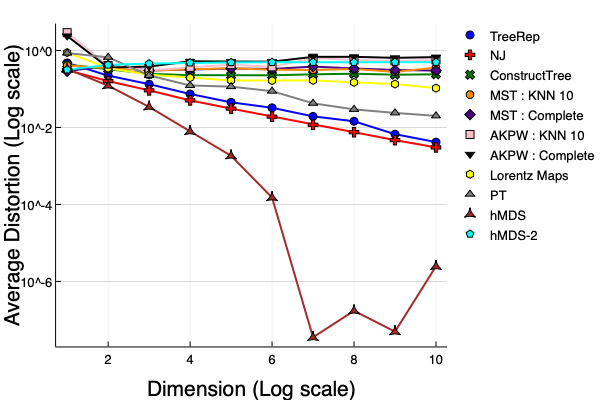}
\label{fig:dim}}\hfill
\subfigure[Varied Scale.]{\includegraphics[width=0.49\linewidth]{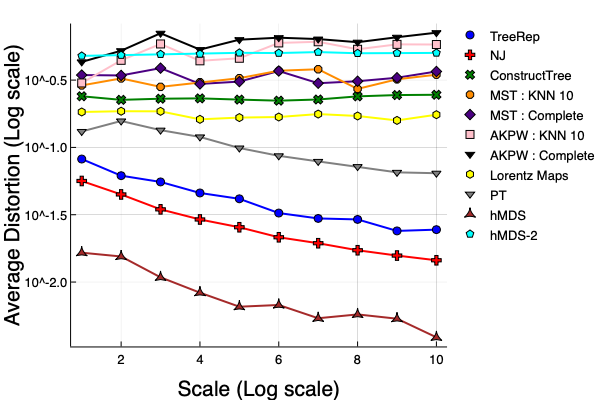}
\label{fig:scale}}
\caption{Average distortion of the metric learned for 100 randomly sampled points from $\mathbb{H}^k$ for $k = 2^i$ and from $\mathbb{H}^{10}$ for scale $s = 2^i$ for $i = 1, 2, \hdots, 10$.}
\end{figure}

\begin{table}[!ht]
\caption{Time taken by PT, LM, hMDS, to learn a 10 dimensional embedding for the synthetic data sets and average time taken by \textsc{TreeRep}  (TR), MST, and CT.}
\centering
\begin{tabular}{cccccccccc}
\toprule
 & TR   & NJ  &  MST & LS & CT  & PT & LM & hMDS & hMDS-2\\ \midrule
Time &  0.002 & 0.06 &\cellbest  0.0001 & 0.002 & 0.076  & 312 & 971 & 11.7 & 0.008 \\
\bottomrule
\end{tabular}
\label{table:hyptime}
\end{table}

\textbf{Random points on Hyperbolic Manifold.}
\label{sec:rand}
We generate two different types of data sets. First, we hold the dimension $k$ constant and scale the coordinates. Second, we hold the magnitude of the coordinates constant and increase the dimension $k$. Note these metrics do not come with an underlying graph! Hence to even apply methods such as MST, or LS we need to do some work. Hence, we create two different weighted graphs; a complete graph and a nearest neighbor graph.

For both types of data, Figures \ref{fig:dim} and \ref{fig:scale} show that as the scale and the dimension increase, the quality of the trees produced by \textsc{TreeRep} and NJ get better. Contrastingly, the quality of the trees produced by MST, \textsc{ConstructTree}, and LS do not improve. Hence we see that when we do not have an underlying sparse graph that was used to generate the metric, methods such as MST and LS do not perform well. In fact, they have the worst performance. This greater generality of possible inputs is one of the major advantages of our method. Thus, demonstrating that \textsc{TreeRep} is an extremely fast algorithm that produces good quality trees that approximate hyperbolic metrics. Furthermore, Table \ref{table:hyptime} shows that \textsc{TreeRep} is a much faster algorithm than NJ. 

For the second task of finding hyperbolic embeddings, we compare against LM, PT and hMDS. For both LM, PT, and hMDS, we compute an embedding into $\mathbb{H}^k$, where $k$ is dimension of the manifold the data was sampled from. We also use hMDS to embed into $\mathbb{H}^2$, we call this hMDS-2. We can see from Figures \ref{fig:dim} and \ref{fig:scale} that \textsc{TreeRep} produces \emph{much better} embeddings than LM, PT, and hMDS-2. Furthermore, LM and PT are extremely slow, with PT and LM taking 312 and 917 seconds on average, respectively. Thus, showing that \textsc{TreeRep} is 5 orders of magnitude faster than LM and PT, \emph{and produces better quality representations.} On the other hand, since our points come from $\mathbb{H}^k$ if we try embedding into $\mathbb{H}^k$ with hMDS we should theoretically have zero error. However, these are high dimensional representations. We want low dimensional hyperbolic representations. Thus, we compared against hMDS-2 which did not perform well. 

\begin{figure}[!htb]
\centering\hfill
\subfigure[TreeRep and NJ Tree]{\includegraphics[width=0.13\linewidth]{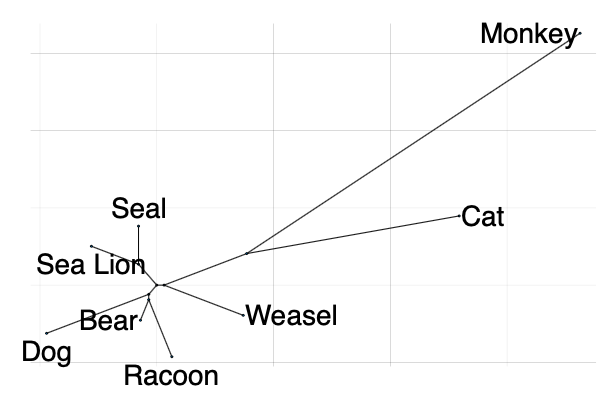}}\hfill
\subfigure[LS Tree]{\includegraphics[width=0.13\linewidth]{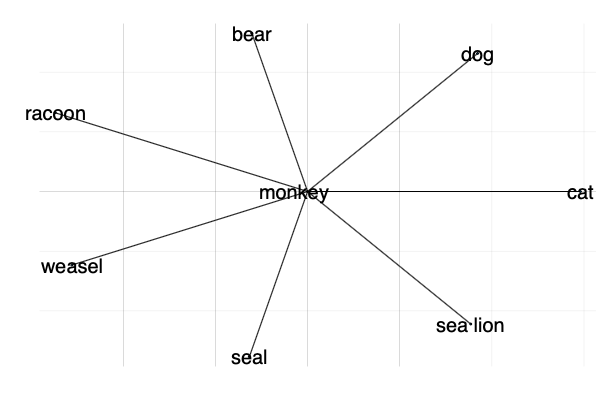}}\hfill
\subfigure[CT Tree]{\includegraphics[width=0.13\linewidth]{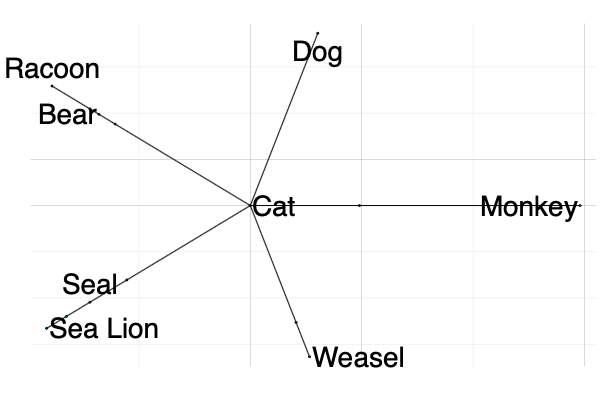}}\hfill
\subfigure[MST Tree]{\includegraphics[width=0.13\linewidth]{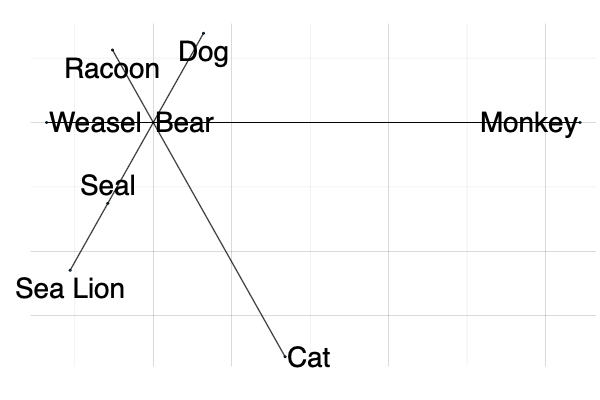}}\hfill
\subfigure[PM Embedding]{\includegraphics[width=0.13\linewidth]{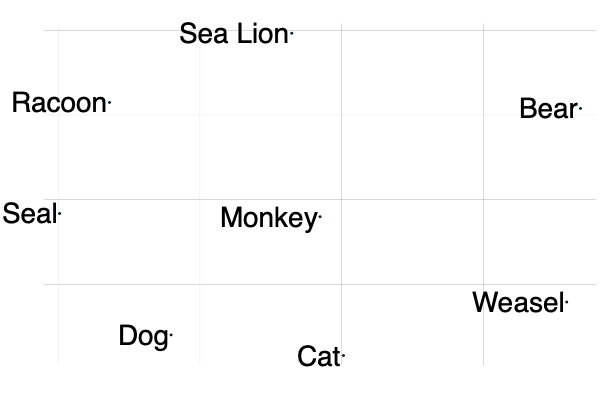}}\hfill
\subfigure[PT Embedding]{\includegraphics[width=0.13\linewidth]{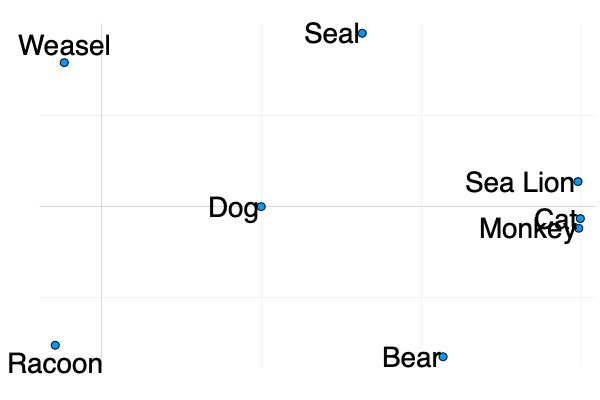}}\hfill
\caption{Tree structure and embeddings for the Immunological distances from \cite{sarich}.}
\label{fig:sarichTree}
\end{figure}

\begin{table}[!htb]
\setlength{\tabcolsep}{3.5pt} 
\caption{Time taken in seconds and the average distortion of the tree metric learned by \textsc{TreeRep}, NJ, MST, and CT and of the 2-dimensional hyperbolic representation learned by PM and PT on the Zeisel and CBMC data set. The numbers for \textsc{TreeRep} (TR) are the average numbers over 20 trials.}
\centering
\begin{tabular}{c|ccccccc||cccc}
\toprule & \multicolumn{7}{c}{Zeisel} & \multicolumn{4}{c}{CBMC} \\
 & TR   & NJ   &  MST & LS & CT  & PT & PM & TR & NJ & MST & LS\\ \midrule
Time &  0.36 & 122.2 &\cellbest 0.11 & 7.2 & >14400 & 8507 & 12342 & 2.8 & >14400 & \cellbest 0.55 & 30\\ 
Distortion & \cellbest0.117 & 0.144& 0.365 & 0.250 & n/a & 0.531 & 0.294 & \cellbest 0.260 & n/a & 1.09 & 1.45\\ 
\bottomrule
\end{tabular}
\label{table:zeisel}
\end{table}

\textbf{Biological Data: scRNA seq and phylogenetic data.}
\label{sec:singlecell}
We also test on three real world biological data sets. The first data set consists of immunological distances from \citet{sarich}. Given these distances, the goal is to recover the hierarchical phylogenetic structure. 
As seen in Figure \ref{fig:sarichTree}, the trees returned by \textsc{TreeRep} and NJ recover this structure well, with sea lion and seal close to each other, and monkey and cat far away from everything else. Divergently, the trees and embeddings produced by MST, LS, \textsc{ConstructTree}, PM, and PT make less sense as phylogenetic trees. 

The second type of data sets are the Zeisel and CBMC sc RNA-seq data set \cite{Zeisel1138, CBMC}. These data sets are expected to be a tree as demonstrated in \citet{bianca}. Here we used the various algorithms to learn a tree structure on the data or to  learn an embedding into $\mathbb{H}^2$. The time taken and the average distortion are reported in Table \ref{table:zeisel}. In this case, we see that \textsc{TreeRep} has the \emph{lowest} distortion. Additionally, \textsc{TreeRep} is 20 times faster than NJ and is 20,000 to 40,000 times faster than PT and PM. Furthermore, NJ, CT, PT, and PM timed out (took greater than 4 hours) on the CBMC data set. For the CBMC data set, we see that \textsc{TreeRep} is \emph{only} algorithm that produces good quality embeddings in a reasonable time frame. Again we see that if the input is a metric instead of a graph, algorithms such as MST and LS do not do well. We also tried to use hMDS for this experiment, but it either didn't output a metric or it outputted the all zero metric. 

\begin{table*}[!h]
\setlength{\tabcolsep}{0.5pt} 
\caption{Table with the time taken in seconds, MAP, and average distortion for all of the algorithms when given metrics that come from unweighted graph. Darker cell colors indicates better numbers for MAP and average distortion. The number next to PT, PM, LM is the dimension of the space used to learn the embedding. The numbers for \textsc{TreeRep} (TR) are the average numbers over 20 trials.} 
\begin{tabular}{lccccccccccccc}
\toprule
Graph &  & TR & NJ & MST & LT & CT & LS  & PT & PT & PM & LM & LM & PM\\ 
& & & & & & & & 2 & 200 & 2 & 2 & 200 & 200 \\\toprule
& \multicolumn{1}{c}{$n$} & \multicolumn{12}{c}{MAP}  \\  \cmidrule(r){3-14}
Celegan & 452 & \cellhim{49} 0.473 & \cellhim{83}  0.713 & \cellhim{31} 0.337 & \cellhim{23} 0.272 & \cellhim{46}  0.447 & \cellhim{28} 0.313 & \cellhim{0} 0.098 & \cellhim{100} 0.857 & \cellhim{50} 0.479 & \cellhim{48}  0.466 &\cellhim{72}  0.646 &\cellhim{74} 0.662   \\ 
Dieseasome & 516 & \cellhim{88} 0.895 &\cellhim{100} 0.962 & \cellhim{70} 0.789 & \cellhim{58} 0.725 & \cellhim{74} 0.815 & \cellhim{69} 0.785 & \cellhim{0} 0.392 &\cellhim{83} 0.868 &\cellhim{71} 0.799 & \cellhim{68} 0.781 & \cellhim{84} 0.874 &  \cellhim{86.6} 0.886   \\ 
CS Phd & 1025 & \cellhim{98}  0.979 & \cellhim{100}  0.993 & \cellhim{100}  0.991 & \cellhim{96}  0.964 &  \cellhim{76}  0.807 & \cellhim{100} 0.991  &\cellhim{0}  0.190 & \cellhim{46} 0.556 & \cellhim{43} 0.537 &\cellhim{43}  0.537 & \cellhim{50} 0.593 & \cellhim{50} 0.593  \\ 
Yeast & 1458 & \cellhim{88}   0.815 & \cellhim{100}  0.892 & \cellhim{97}  0.871 & \cellhim{77}  0.742 & \cellhim{95}  0.859 & \cellhim{97}  0.873 & \cellhim{0}  0.235 & \cellhim{64} 0.658 & \cellhim{48}  0.522 & \cellhim{42} 0.513 & \cellhim{62} 0.641 & \cellhim{62} 0.643 \\ 
Grid-worm  & 3337 & \cellhim{81} 0.707 & \cellhim{100}  0.800 & \cellhim{94}  0.768 & \cellhim{71} 0.657  & - & \cellhim{94} 0.766 & -  & -& \cellhim{6} 0.334 &\cellhim{0}  0.306 & \cellhim{51} 0.558 & \cellhim{50}  0.553 \\
GRQC & 4158 &  \cellhim{54}  0.685  &\cellhim{100}  0.862 & \cellhim{54}  0.686  & \cellhim{0} 0.480  & - & \cellhim{54}  0.684  &- & - & \cellhim{29} 0.589 &\cellhim{32}  0.603 & \cellhim{79}  0.783 &\cellhim{80}  0.784\\
Enron & 33695 &\cellhim{100}  0.570 & - & \cellhim{80}  0.524 & - & - &  \cellhim{80} 0.523  &- & - & - & - & - & - \\
Wordnet & 74374 & \cellhim{99}  0.984  & - & \cellhim{100}  0.989  & - & -&  \cellhim{100}  0.989 &- & - & - & - & - & - \\
\bottomrule

  & $m$ & \multicolumn{12}{c}{Average Distortion}   \\  \cmidrule(r){3-14}
Celegan & 2024 & \cellhim{61} 0.197  & \cellhim{89}  0.124 & \cellhim{38} 0.255 & \cellhim{73} 0.166 &  \cellhim{11} 0.325 & \cellhim{0} 0.353  &\cellhim{46} 0.236 & \cellhim{100} 0.096 & \cellhim{46} 0.236 & \cellhim{40} 0.249 &\cellhim{50} 0.224 &\cellhim{55} 0.211   \\ 
Dieseasome & 1188 & \cellhim{57} 0.188  & \cellhim{61}  0.161 & \cellhim{61} 0.161 & \cellhim{62}  0.157 & \cellhim{7} 0.315 & \cellhim{20} 0.228 & \cellhim{38} 0.227 & \cellhim{100} 0.05 & \cellhim{4} 0.323 & \cellhim{2} 0.328 & \cellhim{0} 0.335 & \cellhim{1} 0.332  \\ 
CS Phd & 1043 & \cellhim{64} 0.204 & \cellhim{89}  0.134 &\cellhim{30}  0.298 & \cellhim{80}  0.161 & \cellhim{36} 0.282 &  \cellhim{37} 0.291  & \cellhim{31} 0.295 & \cellhim{100} 0.105 & \cellhim{2} 0.374 & \cellhim{1} 0.378 & \cellhim{1} 0.378&\cellhim{0} 0.380 \\ 
Yeast & 1948 & \cellhim{40} 0.205 & \cellhim{69}  0.149 & \cellhim{20}  0.243 & \cellhim{20}  0.243 & \cellhim{0} 0.282  & \cellhim{16} 0.243 & \cellhim{27} 0.230 & \cellhim{100} 0.089 & \cellhim{19} 0.246 & \cellhim{18} 0.248 & \cellhim{25} 0.234 & \cellhim{25} 0.234   \\ 
Grid-worm & 6421 & \cellhim{46} 0.188 & \cellhim{100}  0.135 & \cellhim{64} 0.171 & \cellhim{32}  0.202 & - & \cellhim{0} 0.234  & - & - & \cellhim{38} 0.196 & \cellhim{31} 0.203 & \cellhim{42} 0.192 & \cellhim{41} 0.193 \\ 
GRQC & 13422 & \cellhim{100} 0.192 & \cellhim{91}  0.200 & \cellhim{0} 0.275 & \cellhim{10}  0.267 & - & \cellhim{83} 0.206  &- & - & \cellhim{77} 0.212 &\cellhim{94}  0.198 &\cellhim{100}  0.193 & \cellhim{100} 0.193  \\
Enron & 180810 & \cellhim{100} 0.453 & - & \cellhim{20} 0.607 & - & - & \cellhim{49} 0.562  &- & - & - & - & - & - \\
Wordnet & 75834 &\cellhim{77} 0.131 & - & \cellhim{0} 0.336 & - & - &\cellhim{100} 0.071  &- & - & - & - & - & -  \\
\bottomrule

\multicolumn{1}{c}{} & $\delta$ &  \multicolumn{12}{c}{Time in seconds}  \\ \cmidrule(r){3-14}
Celegan & 0.21 &  0.014 & 0.28 &  0.0002  & 0.086 & 0.9 & 0.001&573 & 1156 & 712 & 523 & 1578 &1927 \\ 
Dieseasome & 0.17 &  0.017 & 0.41 &  0.0003 &0.39 & 15.76 & 0.001  & 678  & 1479  & 414 & 365  & 978 &1112 \\ 
CS Phd & 0.23 &  0.037 & 2.94 & 0.0007 & 1.97 & 226 & 0.006  & 1607 &  4145  & 467 & 324 & 768  & 1149  \\ 
Yeast & $\le 0.32$ &  0.057 & 8.04 & 0.0008 &8.21 & 957 & 0.001 & 9526 & 17876 & 972& 619& 1334 &  2269\\ 
Grid-worm & $\le 0.38$ & 0.731 &  163 &  0.001 & 191  & - & 0.007 &- & - & 2645 & 1973 & 4674 &5593\\
GRQC & $\le 0.36$ &  0.42 & 311 &  0.0014  & 70.9 & - & 0.006 &- & - & 7524 & 7217 & 9767 &1187\\
Enron & - &  27 & - & 0.013 & - & - & 0.13 & - &-&-&-&-&- \\
Wordnet & - &  74 & - & 0.18 & - & - & 0.08& - &-&-&-&-&-\\
\bottomrule
\end{tabular}
\label{table:graph}
\end{table*}

\textbf{Unweighted Graphs.}
\label{sec:graphs}
Finally, we consider metrics that come from unweighted graphs. We use eight well known graph data sets from \cite{nr}. Table \ref{table:graph} records the performance of all the algorithms for each of these data sets. For learning tree metrics to approximate general metrics, we see that NJ has the best MAP, with \textsc{TreeRep}, MST, and LS tied for second place. In terms of distortion, NJ is the best, \textsc{TreeRep} is second, while MST is third and LS is sixth. However, \textsc{NJ} is extremely slow and is not viable at scale. Hence, in this case, we have three algorithms with good performance at large scale; \textsc{TreeRep}, MST, and LS. However, MST and LS did not perform well in the previous experiments.

For the task of learning hyperbolic representations, we see that PM, LM, and PT are much slower than the methods that learn a tree first. In fact, these algorithms were too slow to compute the hyperbolic embeddings for the larger data sets. Additionally, this extra computational effort does not always result in improved quality. In all cases, except for the Celegan data set, the MAP returned by \textsc{TreeRep} is superior to the MAP of the $2$-dimensional embeddings produced by PM, LM, and PT. In fact, in most cases, these 2-dimensional embeddings, have worse MAP than all of the tree first methods. Even when they learn 200-dimensional embeddings, PM, LM and PT have worse MAP than \textsc{TreeRep} on most of the data sets. Furthermore, except for PT200, the average distortion of the metric returned by \textsc{TreeRep} is superior to PT2, PM, an LM.
Thus, showing the effectiveness of \textsc{TreeRep} at learning good Hyperbolic representations quickly.

\section{Broader Impact}

There are multiple aspects to the broader impacts of our work, from the impact upon computational biology, specifically, to the impact upon data sciences more generally. The potential impacts on society, both positive and negative, are large. Computational biology is undergoing a revolution due to simultaneous advances in the creation of novel technologies for the collection of multiple and novel sources of data, and in the progress of the development of machine learning algorithms for the analysis of such data. Social science has a similar revolution in its use of computational techniques for the analysis and gathering of data.  

Cellular differentiation is the process by which cells transition from one cell type (typically an immature cell) into more specialized types. Understanding how cells differentiate is a critical problem in modern developmental and cancer biology. Single-cell measurement technologies, such as single-cell RNA-sequencing (scRNA-seq) and mass cytometry, have enabled the study of these processes. To visualize, cluster, and infer temporal properties of the developmental trajectory, many researchers have developed algorithms that leverage hierarchical representations of single cell data. To discover these geometric relationships, many state-of-the-art methods rely on distances in low-dimensional Euclidean embeddings of cell measurements. This approach is limited, however, because these types of embeddings lead to substantial distortions in the visualization, clustering, and the identification of cell type lineages. Our work is specifically focused on extracting and representing hierarchical information.

On the more negative side, these algorithms might also be used to analyze social hierarchies and to divine social structure from data about peoples' interactions. Such tools might encourage, even justify, the intrusive and pervasive collection of data about how people interact and with whom.

\nocite{*}
\bibliography{citations}
\bibliographystyle{plainnat}

\newpage
\section{Metric First Discussion and Justification}
\label{sec:understand}

Table \ref{table:graph} shows that for most of the data sets, learning a tree structure first and then embedding it into hyperbolic space, yields embeddings with better MAP and average distortion compared to methods that learn the embedding directly. One possible explanation for this phenomenon is that the optimization problems that seek the embeddings directly are not being solved optimally. That is, the algorithms get stuck at some local minimum. Another possibility is that there is a disconnect between the objective being optimized and the statistics calculated to judge the quality of the embeddings. 

We propose that there are geometric facts about hyperbolic space that suggest embedding by first learning a tree is the correct approach. The tree-likeness of hyperbolic space has been studied from many different approaches. We present details from \citet{hamann_2018, rtree} and looks at the geometry of $\mathbb{H}^k$ at its two extremes; large scale and small scale. Since $\mathbb{H}^k$ is a manifold, we know that at small scales hyperbolic space looks like Euclidean space. Additionally, in the Poincare disk, the hyperbolic Riemannian metric is given by $\frac{4}{(1-x^2-y^2)^2}(dx^2+dy^2)$ and is just a re-scaling of the Euclidean metric. Thus, at small scales, hyperbolic space is similar to Euclidean space.

Hence to take advantage of hyperbolic representations (i.e., why learn a hyperbolic representation instead of a Euclidean one), we want to embed data into $\mathbb{H}^k$ at scale. To study the large scale geometry of $\mathbb{H}^k$, we consider the asymptotic cone for hyperbolic space $Con(\mathbb{H}^k)$. In particular, we can think of the asymptotic cone as the ``view of our space from infinitely far away''. See the more detailed discussion in Appendix \ref{sec:acones} for examples and complete definitions. The following connects $Con(\mathbb{H}^k)$ to $\mathbb{R}$-tree spaces.

\begin{thm} \cite{Young2008NOTESOA} $Con(\mathbb{H}^k)$ is a complete $\mathbb{R}$-tree. \end{thm}

Thus, we see that the large scale structure of hyperbolic space is a tree, indicating a strong connection between learning trees and learning hyperbolic embeddings. Furthermore, it can be shown that $Con(\mathbb{H}^k)$ is a $2^{\aleph_0}$-universal tree. That is, any tree with finitely many nodes can be embedded into $Con(\mathbb{H}^k)$ exactly. However, these are still embeddings into $Con(\mathbb{H}^k)$. We would like to study embeddings into $\mathbb{H}^k$.

\begin{defn} A metric space $(T , d_T )$ admits an isometric embedding at infinity into the space $(X,d_X)$ if there exists a sequence of positive scaling factors $\lambda_i \to \infty$ such that for every point $t \in T$, there exists an infinite sequence $\{x^i_t\}, i = 1, 2, \hdots $ of points in $X$ such that for all $t_1, t_2 \in T$
$\lim_{i \to \infty} d_X (x^i_{t_1} , x^i_{t_2})/\lambda_i = d_T (t_1, t_2)$ \end{defn}

\begin{thm} \cite{rtree} $Con(\mathbb{H}^k)$ can be isometrically embedded at infinity into $\mathbb{H}^k$. 
\end{thm}

Thus, we can embed any tree into $\mathbb{H}^k$ with arbitrarily low distortion. A type of converse is also true.

\begin{defn} A (geodesic) ray $R$ is a (isometric) homeomorphic image of $[0,\infty)$, such that for any ball $B$ of finite diameter, $R$ lies outside $B$ eventually. \end{defn}

\citet{hamann_2018} showed that we can construct a rooted $\mathbb{R}$-tree $T$ inside $\mathbb{H}^k$, such that every geodesic ray in $\mathbb{H}^k$ eventually converges to a ray of $T$. Thus, showing that any configuration of points at scale in $\mathbb{H}^k$ can be approximated by a tree. Additionally, larger the scale points can be better approximated by trees. More details can be found in Appendix \ref{sec:gtrees}. Thus, showing that learning a tree and then embedding this tree into $\mathbb{H}^k$ is equivalent to learning hyperbolic representations at scale. 

This provides an explanation for why as the scale and dimension increased, \textsc{TreeRep} found a tree that better approximated the hyperbolic metric in Section \ref{sec:rand}. This also provides a justification for why learning a tree first, results in better hyperbolic representations.

\section{Proofs}
\label{sec:proofs}

\subsection{Tree Representation Proofs}
\label{sec:treerepproofs}

{\reflemma{lem:universal} 
Given a metric $d$ on three points $x,y,z$, there exists a (weighted) tree $(T,d_T)$ on four nodes $x,y,z,r$, such that $r$ is adjacent to $x,y,z$, the edge weights are given by $w(x,r) = (y,z)_x$, $w(y,r) = (x,z)_y$ and $w(z,r) = (x,y)_z$, and the metric $d_T$ on the tree agrees with $d$. }
\begin{proof} The basic structure of this tree can be seen in Figure \ref{fig:basicTree}. To prove that the metrics agree we such need to see the following calculation. 
\begin{align*}
d_T(x,y) &= w(x,r) + w(r,y) \\
&= (y,z)_x + (x,z)_y \\
&= \frac{1}{2} (d(x,y) + d(x,z) - d(y,z) \\ &\ \ \ \ \ \ \ + d(x,y) + d(y,z) - d(x,z) )\\
&= d(x,y) \end{align*}
Here $d_T$ is the metric on the tree $T$.
\end{proof}

One important fact that we need is that if $(X,d)$ is a metric graph, then for any three distinct points $x,y,z \in X$, the geodesics connecting them intersect at a unique point. As seen in Lemma \ref{lem:universal}, we refer to this point a Steiner point $r$. It is now important to note that even though $r$ may not be a point in the data set we are given, but $r \in X$ \cite{bridson2013metric}. Thus, in the following lemmas, whenever we find a Steiner point, we will assume that the metric $d$ is defined on $r$. 

\begin{lem} \label{lem:contract} If $d$ is a tree metric and $x,y,w$ are three points then 
\begin{enumerate}[nosep]
\item $(x,y)_w = 0$ if and only if $w \in g(x,y)$
\item $(x,y)_w = d(x,w)$ if and only if $(w,y)_x = 0$. 
\item $(x,y)_w = d(y,w)$ if and only if $(w,x)_y = 0$. 
\end{enumerate} \end{lem}
Here $g(x,y)$ is the unique path connecting $x$ and $y$.
\begin{proof}

For 1. we see that \begin{align*} 0 = (x,y)_w = \frac{1}{2}(d(w,x) + d(w,y) - d(x,y)) \\
\Rightarrow d(x,y) = d(w,x) + d(w,y) \end{align*}
Thus we have that $w \in g(x,y)$. 

For 2. we see that \begin{align*} (x,y)_w = d(x,w) &\Rightarrow d(w,x) + d(w,y) - d(x,y)\\
& \ \ \ \ \ = 2d(w,x) \\
&\Rightarrow d(w,x) +d(x,y) - d(w,y) = 0 \\
&\Rightarrow  2(w,y)_x = 0 \end{align*}

The proof for 3 is similar to that of 2. 

\end{proof}

{\reflemma{lem:structure} 
Let $(X,d)$ be a tree space. Let $w,x,y,z$ be four points in $X$ and let $(T,d_T)$ be the universal tree on $x,y,z$ with node $r$ as the Steiner node. Then we can extend $(T,d_T)$ to $(\hat{T},d_{\hat T})$ to include $w$ such that $d_{\hat T} = d$.
}
\begin{proof}
We note that there are four different possible cases for the configuration of $x,y,z,w$ depending on the relationship amongst the Gromov products. Each case determines a different placement of $r$, as follows:
\begin{enumerate}[nosep]
\item \label{part:r} If $(x,y)_w = (x,z)_w = (y,z)_w = 0$, then replace $r$ with $w$ to obtain $\hat{T}$. 
\item \label{part:zone1r} If $(x,y)_w = (x,z)_w = (y,z)_w = c > 0 $, then connect $w$ to $r$ via an edge of weight $c$ to obtain $\hat{T}$. 
\item \label{part:zone1} If there exists a permutation $\pi:\{x,y,z\} \to \{x,y,z\}$ such that, 
\[ 
    (\pi x,\pi y)_w = (\pi x,\pi z)_w = c < (\pi y,\pi z)_w 
\]  
and $d(\pi x, w) = (\pi x,\pi y)_w$, then connect $w$ to $\pi x$ via an edge of weight $c$ to obtain $\hat{T}$.
\item \label{part:zone2} If there exists a permutation $\pi:\{x,y,z\} \to \{x,y,z\}$ such that, 
\[ 
    (\pi x,\pi y)_w = (\pi x,\pi z)_w = c < (\pi y,\pi z)_w
\]
and $d(\pi x,w) > (\pi x,\pi y)_w$, then add a Steiner point $\hat{r}$ on the edge $x,r$ with $d(\pi x,\hat{r}) = d(\pi x,w) - c$ and connect $w$ to $\hat{r}$ via an edge of weight $c$ to obtain $\hat{T}$. 
\end{enumerate} 

To prove that these extensions of $T$ are consistent, first let us prove that there are exactly four cases. To do that, first note that since we have a  $0$-hyperbolic metric, at least two of the three Gromov products must be equal. Using the triangle inequality, we can see that for any three points $a,b,c$ the following holds 
\[
0 \le (a,b)_c \le d(a,c).
\] 
That is, either we are in the first two cases and three of products are equal, or we have that two of the products are equal. In the case that two of the products are equal, the permutation $\pi$ tells us which of the two are equal and we further subdivide into the case whether $d(\pi x, w) = (\pi x,\pi y)_w$ or $d(\pi x, w) > (\pi x,\pi y)_w$ as we cannot have $d(\pi x, w) < (\pi x,\pi y)_w$. 

Therefore, there are at most four possible configuration cases and it remains to show that the new tree $d_{\hat{T}}$ is consistent with $d$ on the four points. In each case, we present the high level intuition for why these modification result in a consistent tree. The low level details about the metric numbers can easily be checked. 

\textbf{Case 1:} If $(x,y)_w = (x,z)_w = (y,z)_w = 0$, then we replaced $r$ with $w$ in $\hat{T}$. In this case, using Lemma \ref{lem:contract}, we see that $w$ must lie on all tree geodesics $g(x,y), g(x,z), g(y,z)$. Since the metric comes from a tree, these three geodesics can only intersect at one point $r$. Thus, we must replace $r$ with $w$. 

To see that the metric is consistent, we need to verify that $d(w,x)$ = $d_{\hat{T}}(r,x)$. To see we have the following:
\begin{align*}
    d_{\hat{T}}(r,x) &= (y,z)_x \\
                     &= (y,z)_x + (x,y)_w + (x,z)_w - (y,z)_w \\
                     &= d(w,x)
\end{align*}

\textbf{Case 2:} If 
\[
    (x,y)_w = (x,z)_w = (y,z)_w = c > 0, 
\]
then we can see that $(x,w)_r = (y,w)_r = (z,w)_r = 0$. In this case, $r$ lies on geodesics $g(x,y)$, $g(x,z)$, $g(x,w)$, $g(y,w)$, $g(y,z)$, $g(z,w)$. Thus, we must have a star shaped graph with $r$ in the center. 

To see that the metric is consistent we just need to verify that $d(w,x) = d_{\hat{T}}(w,x)$. To see that we have the following calculation. 
\begin{align*}
               d_{\hat{T}}(w,x)         &= d_{\hat{T}}(w,r) + d_{\hat{T}}(x,r) \\
                                        &= (x,y)_w+(y,z)_x \\
                                        &= (x,y)_w + (x,z)_w - (y,z)_w + (y,z)_x \\
                                        &= d(w,x) 
\end{align*}

\textbf{Case 3:} In this case suppose condition \ref{part:zone2} is true. Without loss of generality assume that $\pi$ is the identity map. In each case, we have a tree that looks like a tree in Figure \ref{fig:3trees}. In this case, we can do the calculations and see that $(w,y)_r = (w,z)_r = 0$. That is, the geodesics $g(w,y), g(w,z), g(y,z), g(x,y), g(x,z)$ all intersect at the same point. Thus, again telling us our tree structure.

To check that the metric is consistent, we need to verify that $d(w,y) = d_{\hat{T}}(w,y) = d_{\hat{T}}(w,r) + d_{\hat{T}}(r,y)$. Before we can do that, let us first verify that 

\[ d_{\hat{T}}(w,r) = (y,z)_w \]

To verify this we need to the following calculation 
\begin{align*}
    d_{\hat{T}}(w,r) &= d_{\hat{T}}(r, \hat{r}) + d_{\hat{T}}(\hat{r},w) \\
                     &= c + d_T(x,r) - d_{\hat{T}}(x,\hat{r}) \\
                     &= c + (y,z)_x - (d(x,w)-c) \\
                     &= 2c + (y,z)_x - d(x,w) \\
                     &= (x,y)_w + (x,z)_w + (y,z)_x - d(w,x) \\
                     &= (y,z)_w 
\end{align*}

We then can see that 
\begin{align*}
     d_{\hat{T}}(w,y) & = d_{\hat{T}}(w,r) + d_{\hat{T}}(r,y) \\
                                         &= (y,z)_w + (x,z)_y \\
                                         &= (y,z)_w + (x,y)_w - (x,z)_w + (x,z)_y \\
                                         &= d(w,y)
\end{align*}

Note $d_{\hat{T}}(w,r) = (y,z)_w$ and the consistency of the metric implies that $d(w,r) = (y,z)_w$. Finally, we can see $(w,y)_r = 0$ as follows.

\begin{align*}
    2(w,y)_r &= d(w,r) + d(r,y) - d(w,y) \\
    &= (z,y)_w + (x,z)_y - d(w,y) \\
    &= \frac{1}{2}(d(w,z) - d(w,y) + d(x,y) - d(x,z) \\
    &= (x,z)_w - (x,y)_w \\
    &= 0 
\end{align*}

Note that this also implies that $(w,x)_r > 0$.

\textbf{Case 4:}  In this case, suppose condition \ref{part:zone1} is true. Without loss of generality assume that $\pi$ is the identity map.  Then in this case, we still have that $(w,y)_r = (w,z)_r = 0$, but in addition we have that $(w,y)_x = (w,z)_x = 0$. Thus, again telling us our tree structure. 

In this case, to verify that the metric is consistent, we need to check that $d(w,y) = d_{\hat{T}}(w,y) = d_{\hat{T}}(w,x) + d_{\hat{T}}(x,y)$. To see this we have the following calculations. 

\begin{align*}
     d_{\hat{T}}(w,x) + d_{\hat{T}}(x,y) &= (x,y)_w + d(x,y) \\
                                         &= 2(x,y)_w - (x,z)_w + d(x,y) \\
                                         &= d(w,y) + (w,z)_x
\end{align*}

Thus, now it suffices to show that $(w,z)_x = 0$, which can be seen using the following calculations. 
\begin{align*}
    (x,z)_w = d(w,x) &\Rightarrow 0 = d(x,w)+d(x,z) - d(w,z) \\
                       &\Rightarrow (w,z)_x = 0 
\end{align*}

This also implies that $(w,z)_r = 0$. 
\end{proof}

The proof of Lemma~\ref{lem:structure} shows that there are a number of ways to extend $T$ to include the new point $w$. To clarify our discussion of the extension of $T$, we introduce new terminology.

\begin{defn} \label{defn:zones} Given a data set $V$ (consisting of data points, along with the distances amongst the points), a universal tree $T$ on $x,y,z \in V$ (with $r$ as the Steiner node), let us defining the following three zone types. The first type is associated only with the Steiner node $r$, while the other two types are defined for each of the original nodes $x,y,z$.
\begin{enumerate}[nosep]
    \item $Zone_1(r)$ is all $w \in V$  such that condition \ref{part:zone1r} is true in Lemma \ref{lem:structure}. 
    \item For a given permutation $\pi$, $Zone_1(\pi x)$ is all $w \in V$ such that condition \ref{part:zone1} is true in Lemma \ref{lem:structure} with $\pi$.
    \item For a given permutation $\pi$, $Zone_2(\pi x)$ is all $w \in V$, such that condition \ref{part:zone2} is true in Lemma \ref{lem:structure} with $\pi$.
\end{enumerate}
\end{defn}
Note that there are seven zones total.

\begin{lem} \label{lem:disconnect} Let $(X,d)$ is a metric tree. Let $x,y \in X$ and let $r \in g(x,y)$ if and only if $X \setminus \{r\}$ has at least two disconnected components and $x,y$ are in distinct components. 
\end{lem}
\begin{proof}
Suppose $r \in g(x,y)$. In metric trees, we know that there exist unique simple path between any two points. Therefore, if, after removing $r$, a path connecting $x,y$ remained (i.e., they are in the same component), then there are two simple paths connecting $x,y$ in $X$, which is not possible. 

Suppose $x,y$ are in two separate components of $X\setminus\{r\}$, then because $X$ is path connected, the geodesic between $x$ and $y$ must pass through $r$. 
\end{proof}

{\reflemma{lem:consistency} Given $(X,d)$ a metric tree, and a universal tree $T$ on $x,y,z$, we have the following 
\begin{enumerate}[nosep]
\item If $w \in Zone_1(x)$, then for all $\hat{w} \not\in Zone_1(x)$, we have that $x \in g(w,\hat{w})$. 
\item If $w \in Zone_2(x)$, then for all $\hat{w} \not\in Zone_i(x)$ for $i=1,2$, then we have that $r \in g(w,\hat{w})$. 
\end{enumerate}}

\begin{proof}
First let us prove statement \ref{lem:zone1}. To do this, let us analyze the possible zones to which $\hat{w}$ belongs.

\textbf{Case 1:} Suppose $\hat{w} \in Zone_1(y)$ (similar for $\hat{w} \in Zone_1(z)$). Then we have that $d(\hat{w},y) = (x,y)_{\hat{w}}$. This, implies that $(\hat{w},x)_y = 0$. Thus, by Lemma \ref{lem:contract}, we have that $y \in g(\hat{w},x)$. Similarly we have that $x \in g(w,y)$. 

Now since $w \in Zone_1(x)$, we know that $g(x,w) \cap g(x,y) = \{x\}$. Similarly, know that $g(x,y) \cap g(y, \hat{w}) = \{y\}$. Then using Lemma \ref{lem:disconnect}, on removing $x$, we see that $w$ and $y$ are different connected components. Then since $x \not\in g(\hat{w},y)$, we see that $\hat{w},y$ is in one connected component. Thus, $w, \hat{w}$ are in different components. Thus, $x \in g(w, \hat{w})$ by Lemma \ref{lem:disconnect}.

\textbf{Case 2:} Suppose $\hat{w} \in Zone_2(y)$ (similar for $\hat{w} \in Zone_2(z)$). Now let $r$ be the Steiner node of the universal tree on $x,y,z$. In this case we know from Lemma \ref{lem:structure} that $r \in g(\hat{w},x)$ and that $g(w,x) \cap g(x,r) = \{x\}$. 

Now since $w \in Zone_1(x)$, we know that $g(x,w) \cap g(x,r) = \{x\}$. Similarly, know that $g(x,r) \cap g(r, \hat{w}) = \{r\}$. Then using Lemma \ref{lem:disconnect}, on removing $x$, we see that $w$ and $r$ are different connected components. Then since $x \not\in g(\hat{w},r)$, we see that $\hat{w},r$ is in one connected component. Thus, $w, \hat{w}$ are in different components. Thus, $x \in g(w, \hat{w})$ by Lemma \ref{lem:disconnect}.

\textbf{Case 3:} $\hat{w} \in Zone_2(x)$. Let $r$ be the Steiner node for the universal tree on $x,y,z$. Now my Lemma \ref{lem:structure}, we know that $x \in g(w,r)$. Thus, again by removing $x$ and using Lemma \ref{lem:disconnect}, $r$ and $w$ are in different. We also have that by Lemma \ref{lem:structure} $x \not\in g(\hat{w},r)$. Thus $r, \hat{w}$ are in the same connected component of $X \setminus\{x\}$. Thus, $w$ and $\hat{w}$ are in different connected components. Thus, by Lemma \ref{lem:disconnect}, $x \in g(w,\hat{w})$
 
$ $ 

Thus in all cases, we can see that $x \in g(w,\hat{w})$

$ $

Now let us prove statement \ref{lem:zone2}. Without loss of generality assume that 
\[ 
    \hat{w} \in Zone_i(y)
\]
for $i=1,2$. Then from Lemma \ref{lem:structure}, we know that $r \not\in g(w,x)$ and $r \not\in g(\hat{w},y)$, but $r \in g(x,y)$. Thus, using Lemma \ref{lem:disconnect} on removing $r$, $x$ and $y$ and in different components and $w$ is in the same component as $x$ and $\hat{w}$ is in the same component as $y$. Thus, again using Lemma \ref{lem:disconnect}, we have that $r \in g(w,\hat{w})$. 

\end{proof}

{\reftheorem{thm:metriconstruct} 
Given $(X,d)$, a $\delta$-hyperbolic metric space, and $n$ points $x_1, \hdots, x_n \in X$, \textsc{TreeRep} returns a tree $(T,d_T)$. In the case that $\delta = 0$, $d_T = d$, and $T$ has the fewest possible nodes. \textsc{TreeRep} has worst case run time $O(n^2)$. Furthermore the algorithm is embarrassingly parallelizable.}

\begin{proof} 
The proof of this theorem follows directly from our structural lemmas. More precisely, we show that for $\delta = 0$, \textsc{TreeRep} returns a consistent metric via induction on $n$, the number of data points. 

\textbf{Base Case:} The case when $n \le 3$ is covered by Lemma \ref{lem:universal}. And, the case when $n=4$ is covered by Lemma \ref{lem:structure}. 

\textbf{Inductive Hypothesis:} Assume that for all $k \le n$, our data set of $k$ points is consistent with a $0$-hyperbolic metric $d$, then \textsc{TreeRep} returns a tree $(T,d_T)$ that is consistent with $d$ on the $k$ points. 

\textbf{Inductive Step:} Assume that $w$ is the last vertex attached to $T$. By the inductive hypothesis, we know that without $w$, $(T,d_T)$ is consistent on with $d$ so we only need to show that it is consistent with the addition of $w$. 

Now let $x,y,z$ be the universal tree used to sort $w$ in the penultimate recursive step. Let $r$ be the Steiner node. Then by Lemma \ref{lem:structure}, we know that $d_T(w,x) = d(w,x)$,  $d_T(w,y) = d(w,y)$, and  $d_T(w,z) = d(w,z)$.

Now without loss of generality assume that $w$ was sorted in a zone for $x$. That is, $w \in Zone_i(x)$ for $i=1,2$. 

\textbf{Case 1:} If $w \in Zone_1(x)$. Then from Lemma \ref{lem:zone1}, we know that for all $\hat{w} \not\in Zone_1(x)$, we have that $x \in g(w,\hat{w})$. Thus, having $d_T(x,w) = d(x,w)$ and $d_T(x,\hat{w}) = d(x,\hat{w})$ is sufficient to show consistency. 

Now, since $w$ was placed last there is at most one other point $\tilde{w}$ in $Zone_1(x)$, and $d_T(w,\tilde{w}) = d(w,\tilde{w})$ due to Lemma \ref{lem:universal}. 

\textbf{Case 2:} If $w \in Zone_2(x)$. Then from Lemma \ref{lem:zone2}, we know that for all $\hat{w} \not\in Zone_i(x)$, for $i = 1,2$ we have that $r \in g(w,\hat{w})$. Thus, having $d_T(r,w) = d(r,w)$ and $d_T(r,\hat{w}) = d(r,\hat{w})$ is sufficient to show consistency. 

Suppose $\hat{w} in Zone_1(x)$. Then from Lemma \ref{lem:zone1}, we have that $x \in g(w,\hat{w})$. Thus, having $d_T(x,w) = d(x,w)$ and $d_T(x,\hat{w}) = d(x,\hat{w})$ is sufficient to show consistency.

Finally, since $w$ was the last node placed there are no other nodes in $Zone_2(x)$. 

$ $

Thus, we have the the tree returned by \textsc{TreeRep} is consistent with the input metric $d$.

Notice that whenever we add a Steiner node $r$ we fix the position of at least one data point node. We then look at $O(n)$ Gromov inner products. Thus, we have a worst case running time of $O(n^2)$. 

Additionally, the part where we place nodes into their respective zones can be done in parallel. Thus, if we have $K$ threads then the running time is $O\left(\frac{n^2}{K}\right)$ for the worst running times. 

The final part of the theorem is that we return the tree with the smallest possible nodes. Whenever we look at any triangle formed by three points $x,y,z$, we place a Steiner node $r$. Now, if none of the distances from $x,y,z$ to $r$ is 0, then this Steiner node must exist in all tree consistent with $d$. If one of these distances is 0, we contracted that edge and got rid of $r$. Thus, along with the local consistency argument above this shows that all Steiner nodes that we have placed are necessary (the local consistency argument implies that no two of the Steiner nodes placed could in fact be made into one node due to the nodes beings in different regions). Thus, we have the fewest possible nodes. 

\end{proof}

\subsection{Tree Approximation Proofs}
\label{sec:treeaproxproofs} 

{ \refprop{prop:dist} Given a $\delta$-hyperbolic metric $d$, the universal tree $T$ on $x,y,z$ and a fourth point $w$, when sorting $w$ into its zone $zone_i(\pi x)$, \textsc{TreeRep} introduces an additive distortion of $\delta$ between $w$ and $\pi y, \pi z$ }

\begin{proof}
Without loss of generality assume that $\pi$ is the identity. 
In this case, we know that $d_T(w,r) = (y,z)_w$, and that $d_T(y,r) = (x,z)_y$. Thus, we have the following:

\begin{align*} |d_T(w,y) - d(w,y)| &= |d_T(w,r) + d_T(r,y) - d(w,y)| \\
&= |(y,z)_w + (x,z)_y - d(w,y)| \\
&= \frac{1}{2}|d(w,z) + d(y,x) \\
& \ \ \ \ \ - d(w,y) - d(x,y)| \\
&= |(x,y)_w - (x,z)_w| \\
& \le \delta
\end{align*}
\end{proof}

\section{Geometry: Asymptotic Cones}
\label{sec:acones}

\begin{defn} An ultrafilter $\mathcal{F}$ on $X$ is a subset of $\mathcal{P}(X)$ such that 
\begin{enumerate}[nosep]
\item If $A \in \mathcal{F}$ and $A \subset B$ then $B \in \mathcal{F}$ 
\item $A, B \in \mathcal{F}$ then $A \cap B \in \mathcal{F}$
\item For any $A \subset X$, exactly $1$ of $A, X \setminus A$ is in $\mathcal{F}$
\item $\emptyset \not \in \mathcal{F}$.
\end{enumerate}
\end{defn}

One way to view $\mathcal{F}$ is as defining a probability measure on $X$. In particular, we will view the sets in $\mathcal{F}$ to be large and the sets not in $\mathcal{F}$ to be small. Hence, we can define a measure $\nu$ such that for all $A \in \mathcal{F}$ we have that $\nu(A) = 1$ and for all $A \not \in \mathcal{F}$ we have that $\nu(A) = 0$. 

In this way, we can see that $\nu$ is a finitely additive measure on $X$. One common method to define ultrafilters is to take a point $x \in X$ and let $\mathcal{F}$ be the set of all sets that contain $x$. In this case, the measure $\nu$ has a point mass at $x$ and zero mass elsewhere. Such filters are known an principal ultrafilters. 

Given a measure $\nu$ on $\mathbb{N}$, we can use it to define limits and convergence in $X$. In particular, we have that a sequence $x_i$ converges to $x$, if for all $\epsilon > 0$ we have that \[ \nu\left( \{ x_i : |x_i - x | < \epsilon \} \right) = 1 \] We will denote limits of this form as $\lim_{\nu} x_i = x$. 

We will make use of ultrafilters to construct the asymptotic cone. We will do this via looking at a non-principal ultrafilter on $\mathbb{N}$. We consider non-principal ultrafilters as we want to get a view from infinity, and we do not want to be in the case when one particular index in $\mathbb{N}$ has the entire mass. Hence we restrict ourselves to non-principal ultrafilters. One nice characterization of non-principal ultrafilters is that they are exactly the ultrafilters that have no finite sets.

Now that we have mathematical framework in which we can take limits, let us define our asymptotic cone. Let $\omega$ be a non-principal ultrafilter on $\mathbb{N}$. Let $\{b_i\}_{i \in \mathbb{N}}$ be a sequence of base points and let $\{\lambda_i\}_{i \in \mathbb{N}}$ be a sequence of scaling factors that go to infinity. Let $d$ be the metric on our space $X$. Then let \[ X_{\omega,b_i,\lambda_i} = \{ \{y_i\} : y_i \in X \text{ and } d(b_i,y_i) \le const_{\{y_i\}} \lambda_i \}\]
While this space looks huge we will define an equivalence relation and mod out by this relation to obtain better structure on this space. Given two points $y = \{y_i\}, z = \{z_i\} \in X_{\omega,b_i,\lambda_i}$ we say that $y \sim z$ if \[ \lim_\omega \frac{d(y_i,z_i)}{\lambda_i}  =  0\]

We can now define our asymptotic cone $Con_\omega(X) = X(\omega,b_i,\lambda_i)/\sim$. We can also define a metric on this space as follows, given $y = \{y_i\}, z = \{z_i\} \in Con_\omega(X) $ \[ d_\omega(y,z) := \lim_{\omega} \frac{d(y_i,z_i)}{\lambda_i}   \] 

Let us look at a few examples to get a handle on what $Con_\omega(X)$ looks like. 

\begin{enumerate}
\item Example 1: Let us first consider $X = \mathbb{R}^n$. We know that $\mathbb{R}^n$ is scale invariant. This results in $Con_\omega(\mathbb{R}^n)$ being equivalent to  $\mathbb{R}^n$. In fact, if we assume that $b_i \equiv 0$, then the map $x \mapsto \{\lambda_i x\}$ is an isometry from $\mathbb{R}^n$ to $Con_\omega(\mathbb{R}^n)$

\item Example 2: Suppose $X$ is a bounded metric space. In this case $Con_\omega(X)$ is a single point. 
\end{enumerate}

\begin{defn} A metric space $(X,d_x)$ can be isometrically embedded into a metric space $(Y,d_y)$ if there exists a map $f : X \to Y$ such that for all $x_1, x_2 \in X$ we have that 
\[
	d_x(x_1,x_2) = d_y(f(x_1),f(x_2))
\]
Such a map $f$ is known as an isometry. 
\end{defn}

\begin{defn} A metric space $(X,d)$ is homogenous if for all $x, y \in X$ there exists an isometry $f: X \to X$ such that $f(x) = y$. \end{defn}

\begin{defn} Given a $\mathbb{R}$-tree $T$, the valency of a point $x \in T$ in an $\mathbb{R}$-tree is the number of connected components in $T \setminus \{x\}$. Let the valence of a the tree, denoted $val(T)$, be the maximum valence of any point in $T$. \end{defn}

\begin{defn} A $\mathbb{R}$-tree $T$ is a $\mu$-universal if every $\mathbb{R}$-tree $\hat{T}$ with $val(\hat{T}) \le \mu$ can be isometrically embedded into $T$. \end{defn}

Here we can see that we can embed any finite tree into a $2^{\aleph_0}$-universal tree $T$. Hence, if could isometrically embed $T$ into $Con(\mathbb{H}^n)$ then we can embed any tree into $Con(\mathbb{H}^n)$. This and more turns out to be true. 

\begin{thm} \cite{rtree} Any $2^{\aleph_0}$-universal $\mathbb{R}$-tree can be isometrically embedded into the asymptotic cone for any complete simply connected manifold of negative curvature.
\end{thm}

\section{Geometry: Geodetic Tree}
\label{sec:gtrees}

In general, it is rare to be able isometrically embed one space into another. Hence, we have the following weaker definition. 

\begin{defn} We say that we can quasi isometrically embed a metric space $(X,d_x)$ into a metric space $(Y,d_y)$ if there exists a map $f : X \to Y$ and real numbers $c, \lambda \in \mathbb{R}$ such that $\lambda \ge 1$, $c > 0$ and for all $x_1, x_2 \in X$ we have that
\[
	\frac{1}{\lambda} d_x(x_1,x_2) - c  \le d_y(f(x_1),f(x_2)) \le \lambda d_x(x_1,x_2) + c 
\]
\end{defn}

Such isometries are called $(\lambda, c)$-quasi-isometries. 

It is has been shown that any $\delta$-hyperbolic metric space $(X,d)$ with bounded growth admits a quasi-isometric embedding into $\mathbb{H}^k$ \cite{Bonk2000}.

\begin{defn} We say that a ray $R$ is quasi geodetic if instead of being an isometric image of $[0,\infty)$, we have that $R$ is an quasi-isometric image of $[0,\infty)$. \end{defn}

\begin{defn} A ray is eventually (quasi) geodetic if it has a subray that is (quasi) geodetic. \end{defn}

\begin{thm} \label{thm:geodesic1} \cite{hamann_2018} For all $\lambda \ge 1$,$c \ge 0$ there is a constant $\kappa = \kappa(\delta, \lambda, c)$, such that for every two points $x, y \in \mathbb{H}^k$, every $(\lambda, c)$-quasi-geodesic between them lies in a $\kappa$-neighborhood around every geodesic between $x$ and $y$ and vice versa.\end{thm}

\begin{defn} \label{def:equiv} Two geodetic rays $\pi_1,\pi_2$ are equivalent if for any sequence $(x_n)$ of points on $\pi_1$, we have $\liminf_{n\to \infty} d(x_n, \pi_2) \le M$ for an $M < \infty$ \end{defn}

\begin{defn} The boundary $\partial \mathbb{H}^k$ of $\mathbb{H}^k$ is the equivalence class of all geodesic rays. \end{defn}

\begin{thm}\label{thm:geodesic2} \cite{hamann_2018} There is an $\mathbb{R}$-tree $T \subset \mathbb{H}^k$ such that the canonical map $\gamma$ from $\partial T$ to $\partial X$ exists and has the following properties.
\begin{enumerate}
    \item It is surjective;
    \item there is a constant $M < \infty$ depending only on $k$ such that $\gamma^{-1}(\eta)$ has at most $M$ elements for each $\eta \in \partial \mathbb{H}^k$. 
\end{enumerate}
\end{thm}

\begin{thm} \cite{hamann_2018} Let $T$ be the $\mathbb{R}$-tree in Theorem \ref{thm:geodesic2} with root $r$. There exist constants $\lambda \ge 1$, $c \ge 0$ such that every ray in $T$ starting at the root is a $(\lambda,c)$-quasi-geodetic ray in $\mathbb{H}^k$.
\end{thm}

The above two theorems tell us that given any geodesic ray $R$ in $\mathbb{H}^k$ there is exists a ray in $T$ that is equivalent to $R$ (via $\sim$ in Definition \ref{def:equiv}). Furthermore this ray in $T$ is $(\lambda,c)$-quasi-geodetic ray in $\mathbb{H}^k$. Thus, due to Theorem \ref{thm:geodesic1} any configuration of points at scale in $\mathbb{H}^k$ can be approximated by a tree such that the larger the scale, better the approximation. 

\section{\textsc{TreeRep} Best}
\label{sec:trbest}

So far all numbers for the \textsc{TreeRep} algorithm that we have reported are averages. But due to the speed of the algorithm, we can actually run the experiment multiple times and pick the tree with the best metric. 

\begin{table*}[!ht]
\caption{\textsc{TreeRep} Best Numbers}
\centering
\begin{tabular}{ccccccc}
\toprule
& \multicolumn{2}{c}{No Opt} & \multicolumn{2}{c}{Heuristic Opt} & \multicolumn{2}{c}{Full Opt} \\ \cmidrule(r){2-3}\cmidrule(r){4-5}\cmidrule(r){6-7}
Graph & MAP & Distortion & MAP & Distortion & MAP & Distortion  \\ \midrule
Celegan & 0.508 & 0.173 & 0.539 & 0138 & 0.547  &   0.119\\
Diseasome & 0.912 & 0.134 & 0.911 & 0.106 &  0.890 &  0.092\\
CS PhD & 0.987  & 0.134 &  0.984 &  0.119 &  0.968 & 0.121 \\ 
Yeast &  0.841 & 0.171 & 0.833 & 0.150 & 0.808 &  0.135 \\ 
Grid-worm  & 0.727  & 0.154 & 0.728 &  0.125 & - & - \\
GRQC & 0.699 & 0.175 &0.694  & 0.152 & - & - \\
\bottomrule 
\end{tabular}
\label{table:best}
\end{table*}

\section{Improving Distortion}
\label{sec:improvedist}

We have seen that in the case of unweighted graphs \textsc{TreeRep} produces better MAP than PM, LM, and PT. However, PT tends to have better average distortion. Hence, we want to be able to improve the distortion. Once we have learned the tree structure we can set up an optimization problem to learn the edge weights on the tree to improve the distortion. Specifically, since the metric comes from the tree, for any pair of data points, there is exactly one path connecting the two data points. Thus, regardless of the edges weights, this path is the shortest path between the data points. Thus, we can set up an optimization problem of the following form:

\[ \argmin_{w} \|AW - D\|_2. \]

Here $W$ is a vector containing the edge weights, $D$ is a vector containing the original metric, and $A$ is a matrix that encodes all of the paths. This optimization problem however, is unfeasible as $n$ gets longer. So instead we sample some rows of $A$ and solve a heuristic problem. As can be seen from Table \ref{table:heurestic}, we are still faster than NJ but now have improved our distortion without sacrificing MAP. 

\begin{table}[!htbp]
\setlength{\tabcolsep}{3pt} 
\caption{MAP and average distortion for the \textsc{TreeRep} and MST after doing the heuristic optimization. The time taken for both optimizations is the same.}
\centering
\begin{tabular}{cccccc}
\toprule
Graph & Time & Distortion & MAP & Distortion & MAP \\ \midrule
\multicolumn{2}{c}{} &  \multicolumn{2}{c}{\textsc{TreeRep}}  &  \multicolumn{2}{c}{MST} \\ \cmidrule(r){3-4} \cmidrule(r){5-6}
Celegans & 0.69 & 0.157 & 0.504 & 0.195 & 0.357 \\
Diseasome & 1.56 & 0.121 & 0.891 & 0.111 & 0.774  \\
CS Phd & 1.2 & 0.152 &  0.971 & 0.170 & 0.989  \\
Yeast & 4.2 & 0.163 & 0.813 & 0.171  & 0.862 \\
Grid Worm & 32 & 0.164 &  0.707 & 0.151 & 0.768 \\
GRQC & 68 & 0.157 &  0.676 & 0.159 & 0.669 \\
\bottomrule
\end{tabular}
\label{table:heurestic}
\end{table}

\section{Experiment and Practical Details}
\label{sec:expappendix}

\subsection{MAP and Average Distortion}

\begin{defn} \label{def:dist} Given two metrics $d_1, d_2$ on a finite set $X = {x_1, \hdots, x_n}$ the average distortion is:
\[
	\frac{1}{\binom{n}{2}} \sum_{i=1}^n \sum_{j < i} \frac{|d_1(x_i,x_j) - d_2(x_i,x_j)|}{d_2(x_i,x_j)} 
\]
Smaller average distortion implies greater similarity between $d_1$ and $d_2$.
\end{defn}

In many cases, the metric learned by the various algorithms will be a scalar multiple of the actual metric, so we will solve for the scale $\alpha := \argmin_c \| D - c\hat{D}\|_F$, before calculating the average distortion.\footnote{For NJ and LT, computing this $\alpha$ made the average distortion worse, so we report numbers un-scaled. Additionally, computing $\alpha$ is too computationally expensive for bigger data sets and was not done for the Enron and Wordnet data set.}

\begin{defn} \label{def:map} Let $d$ be a metric on the nodes of a graph $G = (V,E)$. For $v \in V$, let $N(v) = \{u_1, \hdots, u_{deg(v)}\}$ be the neighborhood of $v$. Then let $B_{v,u_i} = \{ u \in V \setminus \{u\} : d(u,v) \le d(v,u_i) \}$. Then the mean average precision (MAP) is defined to be 
\[ 
	\frac{1}{n} \sum_{v \in V} \frac{1}{deg(v)} \sum_{i=1}^{|N(v)|} \frac{|N(v) \cap B_{v,u_i}|}{|B_{v,u_i}|} 
\]
Closer MAP is to 1, the closer $d$ is to approximating $d_G$.
\end{defn}

\subsection{TreeRep}

There are a few practical details that must be discussed in relation to the \textsc{TreeRep} algorithm. 

\begin{enumerate}[nosep]
\item Pre-allocate the matrix for the weights of edges of the tree as a dense matrix. Doing this greatly speeds up computations. Note the proof of Lemma \ref{lem:structure}, show that we need at most $n$ Steiner nodes. Thus, the tree has about $2n$ nodes. Since the input to the algorithm is a dense $n \times n$ matrix, we already need $O(n^2)$ memory. Thus, having a dense $2n \times 2n$ matrix is still linear memory usage in the size of the input. 
\item When doing zone 2 recursions pick the node closest to $r$ as the new $z$ as suggested by Proposition \ref{prop:dist}. 
\item The placement of nodes into their respective zones can be done in parallel. For all of the experiments in the paper, we used 8 threads to do the placement for all of the experiments, except that we used 1 thread for the random points from $\mathbb{H}^k$ experiment and for CBMC experiment. 
\item All of the numbers reported are averages over 20 iterations. We could have also picked the best over 20 iterations as our algorithm is fast enough for this to be viable. 
\item When checking for equality, instead of checking for exact equality, we checked whether two numbers are within 0.1 of each other.
\item It is possible for some of the edge weights to be set to a negative number. In this case, after the algorithm terminated we set those edge weights to 0. 
\end{enumerate}

\subsection{Bartal}

We sample 200 trees from the distribution and compute the metric assuming that we are embedding into the distribution restricted to these 200 trees. 

\subsection{Neighbor Join}

The following implementation of NJ was used: http://crsl4.github.io/PhyloNetworks.jl/latest/. We set the options so as to not have any negative edge weights. 

\subsection{MST} 

Prim's algorithm for calculating MST was used. We used the implementation at https://github.com/JuliaGraphs/LightGraphs.jl

\subsection{LS}

Low stretch spanning trees are calculated using Laplacian package in Julia. This code is based an adaptation of \cite{alon} by the authors of \cite{elkin}.

\subsection{LevelTree and ConstructTree}

To the best of the authors knowledge there does not exist a publicly available implementations of these algorithms. Both of these algorithms were implemented by the authors.

Note that LevelTree claims to be a $O(n)$ algorithm, but this only true, once we have calculated the sphere $S_n$ needed for the algorithm. However, it takes $O(n^2)$ time to calculate the spheres $S_n$ (equivalent to solving single source all destination shortest path problem). 

\subsection{PM and LM}
The following options were used. The number of epochs was to set to be higher than default. Everything else was left at default. One note about PM and LM is that their objective function is set up to optimize for MAP and not average distortion. 

\begin{enumerate}[nosep]
\item -lr 0.3
\item -epochs 1000
\item -burnin 20
\item -negs 50
\item -fresh 
\item -sparse 
\item -train\_threads 2
\item -ndproc 4 
\item -batchsize 10
\end{enumerate}

For PM we used \texttt{-manifold poincare}, for LM we used \texttt{-manifold lorentz}. The code is taken from https://github.com/facebookresearch/poincare-embeddings

\subsection{PT}

The following options were used. We used the \texttt{--learn-scale} option as based on the discussion in the appendix of \citet{albert} learning the scale results in better quality metrics. Additionally, we add a burnin phase to the optimization. Finally, based on the discussion in \cite{albert}, the objective function for PT has a lot of shallow local minimas. Thus, we added momentum and used Adagrad for the optimization to try and avoid these local minimums. 

\begin{enumerate}[nosep]
\item --learn-scale
\item --burn-in 100
\item --momentum 0.9
\item --use-adagrad
\item --l 5.0
\item --epochs 1000
\item --batch-size 256
\item --subsample 64 
\end{enumerate}

The code is taken from https://github.com/HazyResearch/hyperbolics

\subsection{Hardware}

All experiments were run on Google cloud instances. For PM, LM and PT we created a fresh instance for each algorithm. Each instance for an algorithm only had the bare minimum installed to run those algorithms. We used \texttt{n1-highmem-8} instances. The specification of each of the instances are as follows:

\begin{enumerate}[nosep]
    \item 8 cores each with 6.5 GB of ram. 
    \item Ubuntu-1604-xenial-v20190913 operting system.
    \item 100 standard persistent disk. 
\end{enumerate}

For TreeRep, NJ, CT, LT and MST, we ran all code via a Jupyter notebook interface running Julia 1.1.0. All experiments (except for the experiments with Enron and Wordnet), we done on instances with the same specification as above. 

For Enron and Wordnet, we need more memory to store the distance matrices. Thus, used an since with the following specifications.

\begin{enumerate}[nosep]
    \item 24 cores each with 6.5 GB of ram. 
    \item Ubuntu-1604-xenial-v20190913 operting system.
    \item 100 standard persistent disk. 
\end{enumerate}

\subsection{Synthetic $0$-hyperbolic metrics}

To produce random synthetic $0$-hyperbolic metrics, we do the following. First, we take a complete binary tree of depth $i$. We then compute its double tree. Then for each node in this tree we sample a number $C$ from $2$ to $10$ and replace the node with a clique of size $C$. We then pick a random node in the tree and compute the breadth first search tree from that node. We then assign edge uniformly randomly, sampled from $[0,1]$. 

\subsection{Synthetic Data Sets}

Here we sampled coordinates from the standard normal $\mathcal{N}(0,1)$. The final coordinate $x_0$ is set so that the point lies on the hyperboloid manifold. In the presence of a scale we just multiplied each coordinate by that scale before calculating $x_0$. We ran \textsc{TreeRep} with 1 thread. 

\subsection{Phylogenetic and Single Cell Data}

The immunological distances can be seen in Figure \ref{fig:sarich}. The matrix is symmeterized by averaging across the diagonal. In this case, we ran \textsc{TreeRep} 10 times and picked the tree with the lowest average distortion.

The figures for the trees are produced using an adaptation of Sarkar's construction for Euclidean space. The code from PT also produces a picture. This picture can be seen in Figure \ref{fig:PTSarich}. As we can see, this figure is similar to the one in the main text. 

\begin{figure}
    \centering
    \includegraphics[width = 0.8\linewidth]{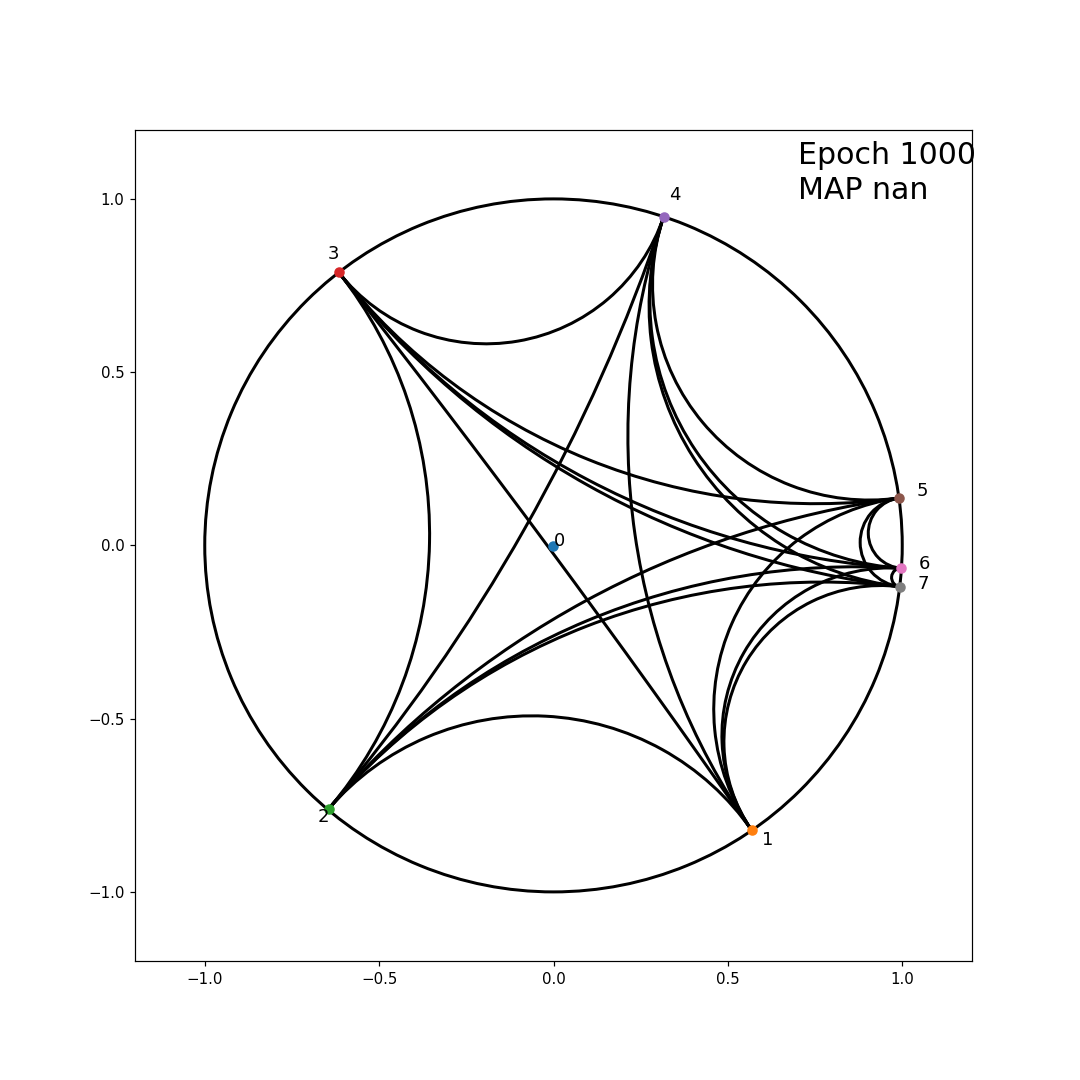}
    \caption{Figure for Sarich data produced by PT code}
    \label{fig:PTSarich}
\end{figure}

\begin{figure}[!htbp]
\centering
\includegraphics[width=0.6\linewidth]{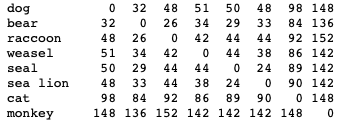}
\caption{Immunological distances from \cite{sarich}}
\label{fig:sarich}
\end{figure}

For the Zeisel data we did the same pre-processing as done in \citet{bianca}. For PM and MST, we use 10 nearest neighbor graph. For LS we used the complete graph. 

For the CBMC data we did the same pre-processing as done in \citet{bianca}. For MST and LS we used the complete graph.

\subsection{Unweighted Graphs}

Some of the graphs are disconnected. The largest connected component of each graph was used.

For $\delta$ calculation, we normalized the distances so that the maximum distance was 1 and then calculated $\delta$. For Celegans, Diseasome, and Phds,, this calculation is exact. 

For Yeast, Grid-worm and GRQC, we fixed the base point to be $w = 1$ and then calculated $\delta$. It is known from theory that for any fixed base point the $\delta$ is at least half of the $\delta$ for the whole metric \cite{bridson2013metric}. Thus, we get the inequality. 

All experiments with a ``-'' were terminated after 4 hours.

\begin{table*}[!ht]
\setlength{\tabcolsep}{1pt} 
\caption{Table with the time taken in seconds, MAP, and average distortion for all of the algorithms when given metrics that come from unweighted graph. Darker cell colors indicates better numbers for MAP and average distortion. The number next to PT, PM, LM is the dimension of the space used to learn the embedding. The numbers for \textsc{TreeRep} (TR) are the average numbers over 20 trials. The table also shows some graph statistics such as $n$, the number of nodes, $m$, the number of edges, and $\delta$, the hyperbolcity of the metric.}
\begin{tabular}{lcccccccccccccc}
\toprule
Graph &  & TR & NJ & MST & LT & CT & LS & Bartal & PT & PT & PM & LM & LM & PM\\ 
& & & & & & & & & 2 & 200 & 2 & 2 & 200 & 200 \\\toprule
& \multicolumn{1}{c}{$n$} & \multicolumn{13}{c}{MAP}  \\  \cmidrule(r){3-15}
Celegan & 452 & \cellhim{49} 0.473 & \cellhim{83}  0.713 & \cellhim{31} 0.337 & \cellhim{23} 0.272 & \cellhim{46}  0.447 & \cellhim{28} 0.313 & \cellhim{42}  0.436 & \cellhim{0} 0.098 & \cellhim{100} 0.857 & \cellhim{50} 0.479 & \cellhim{48}  0.466 &\cellhim{72}  0.646 &\cellhim{74} 0.662   \\ 
Dieseasome & 516 & \cellhim{88} 0.895 &\cellhim{100} 0.962 & \cellhim{70} 0.789 & \cellhim{58} 0.725 & \cellhim{74} 0.815 & \cellhim{69} 0.785 &  \cellhim{38} 0.610& \cellhim{0} 0.392 &\cellhim{83} 0.868 &\cellhim{71} 0.799 & \cellhim{68} 0.781 & \cellhim{84} 0.874 &  \cellhim{86.6} 0.886   \\ 
CS Phd & 1025 & \cellhim{98}  0.979 & \cellhim{100}  0.993 & \cellhim{100}  0.991 & \cellhim{96}  0.964 &  \cellhim{76}  0.807 & \cellhim{100} 0.991  & \cellhim{0}   0.190&\cellhim{0}  0.190 & \cellhim{46} 0.556 & \cellhim{43} 0.537 &\cellhim{43}  0.537 & \cellhim{50} 0.593 & \cellhim{50} 0.593  \\ 
Yeast & 1458 & \cellhim{88}   0.815 & \cellhim{100}  0.892 & \cellhim{97}  0.871 & \cellhim{77}  0.742 & \cellhim{95}  0.859 & \cellhim{97}  0.873 &- & \cellhim{0}  0.235 & \cellhim{64} 0.658 & \cellhim{48}  0.522 & \cellhim{42} 0.513 & \cellhim{62} 0.641 & \cellhim{62} 0.643 \\ 
Grid-worm  & 3337 & \cellhim{81} 0.707 & \cellhim{100}  0.800 & \cellhim{94}  0.768 & \cellhim{71} 0.657  & - & \cellhim{94} 0.766 &- & -  & -& \cellhim{6} 0.334 &\cellhim{0}  0.306 & \cellhim{51} 0.558 & \cellhim{50}  0.553 \\
GRQC & 4158 &  \cellhim{54}  0.685  &\cellhim{100}  0.862 & \cellhim{54}  0.686  & \cellhim{0} 0.480  & - & \cellhim{54}  0.684 & - &- & - & \cellhim{29} 0.589 &\cellhim{32}  0.603 & \cellhim{79}  0.783 &\cellhim{80}  0.784\\
Enron & 33695 &\cellhim{100}  0.570 & - & \cellhim{80}  0.524 & - & - &  \cellhim{80} 0.523 & - &- & - & - & - & - & - \\
Wordnet & 74374 & \cellhim{99}  0.984  & - & \cellhim{100}  0.989  & - & -&  \cellhim{100}  0.989 &- &- & - & - & - & - & - \\
\bottomrule

  & $m$ & \multicolumn{13}{c}{Average Distortion}   \\  \cmidrule(r){3-15}
Celegan & 2024 & \cellhim{61} 0.197  & \cellhim{89}  0.124 & \cellhim{38} 0.255 & \cellhim{73} 0.166 &  \cellhim{11} 0.325 & \cellhim{0} 0.353 &\cellhim{51} 0.220 &\cellhim{46} 0.236 & \cellhim{100} 0.096 & \cellhim{46} 0.236 & \cellhim{40} 0.249 &\cellhim{50} 0.224 &\cellhim{55} 0.211   \\ 
Dieseasome & 1188 & \cellhim{57} 0.188  & \cellhim{61}  0.161 & \cellhim{61} 0.161 & \cellhim{62}  0.157 & \cellhim{7} 0.315 & \cellhim{20} 0.228 & \cellhim{2} 0.330 & \cellhim{38} 0.227 & \cellhim{100} 0.05 & \cellhim{4} 0.323 & \cellhim{2} 0.328 & \cellhim{0} 0.335 & \cellhim{1} 0.332  \\ 
CS Phd & 1043 & \cellhim{64} 0.204 & \cellhim{89}  0.134 &\cellhim{30}  0.298 & \cellhim{80}  0.161 & \cellhim{36} 0.282 &  \cellhim{37} 0.291 & \cellhim{20} 0.326 & \cellhim{31} 0.295 & \cellhim{100} 0.105 & \cellhim{2} 0.374 & \cellhim{1} 0.378 & \cellhim{1} 0.378&\cellhim{0} 0.380 \\ 
Yeast & 1948 & \cellhim{40} 0.205 & \cellhim{69}  0.149 & \cellhim{20}  0.243 & \cellhim{20}  0.243 & \cellhim{0} 0.282  & \cellhim{16} 0.243 &- & \cellhim{27} 0.230 & \cellhim{100} 0.089 & \cellhim{19} 0.246 & \cellhim{18} 0.248 & \cellhim{25} 0.234 & \cellhim{25} 0.234   \\ 
Grid-worm & 6421 & \cellhim{46} 0.188 & \cellhim{100}  0.135 & \cellhim{64} 0.171 & \cellhim{32}  0.202 & - & \cellhim{0} 0.234 &- & - & - & \cellhim{38} 0.196 & \cellhim{31} 0.203 & \cellhim{42} 0.192 & \cellhim{41} 0.193 \\ 
GRQC & 13422 & \cellhim{100} 0.192 & \cellhim{91}  0.200 & \cellhim{0} 0.275 & \cellhim{10}  0.267 & - & \cellhim{83} 0.206 & - &- & - & \cellhim{77} 0.212 &\cellhim{94}  0.198 &\cellhim{100}  0.193 & \cellhim{100} 0.193  \\
Enron & 180810 & \cellhim{100} 0.453 & - & \cellhim{20} 0.607 & - & - & \cellhim{49} 0.562 & - &- & - & - & - & - & - \\
Wordnet & 75834 &\cellhim{77} 0.131 & - & \cellhim{0} 0.336 & - & - &\cellhim{100} 0.071 & - &- & - & - & - & - & -  \\
\bottomrule

\multicolumn{1}{c}{} & $\delta$ &  \multicolumn{13}{c}{Time in seconds}  \\ \cmidrule(r){3-15}
Celegan & 0.21 &  0.014 & 0.28 &  0.0002  & 0.086 & 0.9 & 0.001 & 226 &573 & 1156 & 712 & 523 & 1578 &1927 \\ 
Dieseasome & 0.17 &  0.017 & 0.41 &  0.0003 &0.39 & 15.76 & 0.001  & 313 & 678  & 1479  & 414 & 365  & 978 &1112 \\ 
CS Phd & 0.23 &  0.037 & 2.94 & 0.0007 & 1.97 & 226 & 0.006 & 3559 & 1607 &  4145  & 467 & 324 & 768  & 1149  \\ 
Yeast & $\le 0.32$ &  0.057 & 8.04 & 0.0008 &8.21 & 957 & 0.001 &- & 9526 & 17876 & 972& 619& 1334 &  2269\\ 
Grid-worm & $\le 0.38$ & 0.731 &  163 &  0.001 & 191  & - & 0.007 & - &- & - & 2645 & 1973 & 4674 &5593\\
GRQC & $\le 0.36$ &  0.42 & 311 &  0.0014  & 70.9 & - & 0.006 &- & - & - & 7524 & 7217 & 9767 &1187\\
Enron & - &  27 & - & 0.013 & - & - & 0.13 & - &-&-&-&-&-&- \\
Wordnet & - &  74 & - & 0.18 & - & - & 0.08& - &-&-&-&-&-&-\\
\bottomrule
\end{tabular}
\label{table:graph}
\end{table*}

\subsection{Calculating $\alpha$}

Can be calculated directly using 

\[
	\alpha = \frac{Tr(D'*D}{\|D\|_F^2} 
\]

\section{Tree Representation Pseudo-code}
\label{sec:code}

\begin{algorithm}[!htb]
\caption{Recursive parts of TreeRep.}
\label{alg:tree_struct_recurse}
	\begin{algorithmic}[1]
	\Function {\textsc{zone1\_recursion}}{$T$, $d_T$, $d$, $L$, $v$}
		\If{Length($L$) == 0}
			\State \textbf{return} $T$
		\EndIf
		\If{Length($L$) == 1}
			\State Set $u$ = pop($L$) and add edge $(u, v)$ to $E$
			\State Set edge weight $d_T(u,v) = d(u,v)$ 
			\State \textbf{return} $T$
		\EndIf
		\State Set $u = $pop($L$), $z = $pop($L$)
		\State \textbf{return} \textsc{recursive\_step}($T,L,v,u,z,d$, $d_T$)
	\EndFunction \\
	
	\Function {\textsc{zone2\_recursion}}{$T$, $d_T$, $d$, $L$, $u$, $v$}
		\If{Length($L$) == 0}
			\Return $T$
		\EndIf
		\State Set $z$= the closest node to $v$. 
		\State Delete edge $(u,v)$
		\State \textbf{return:} \textsc{recursive\_step}($T,L,v,u,z,d$, $d_T$)
	\EndFunction
	\end{algorithmic}
\end{algorithm}

\begin{algorithm}[!ht]
\caption{Metric to tree structure algorithm.}
\label{alg:tree_struct2}
	\begin{algorithmic}[1]
	\Function {\textsc{Tree structure}}{X, $d$}
		\State $T = (V,E,d') = \emptyset $ \quad
		\State Pick any three data points uniformly at random $x,y,z \in X$. 
		\State $T$ = \textsc{recursive\_step}($T,X,x,y,z,d,d_T,$)
		\State \textbf{return} $T$
	\EndFunction \\
	\State
	\Function {\textsc{recursive\_step}}{$T,X,x,y,z,d,d_T,$}
		\State Let $Z1(r \rightarrow  [], x \rightarrow [], y\rightarrow [], z\rightarrow []), Z2(x \rightarrow [], y\rightarrow [], z\rightarrow [])$ \quad // Dictionaries of list for various zones
		\State Place an additional node $r$ in $V$ and add edges $xr, yr, zr$ to $E$ 
		\State Set the weights $d_T(x,r) = (y,z)_x$, $d_T(y,r) = (x,z)_y$, and $d_T(z,r) = (x,y)_z$ // If edge weight = 0, contract the edge. 
		\For{all remaining data points $w \in X$} 
			\State  $a = (x,y)_{w}$, $b = (y,z)_{w}$, $c = (z,x)_{w}$, $m = 0$, $m2 = 0$
			\If{$a == b == c$}
				\State push($w$, $Z1[r]$)
				\State Set $d_T(w,r) = (x,y)_{w}$
			\ElsIf{$a == maximum(a,b,c)$}
				\State $\pi = (x \rightarrow z, y \rightarrow y, z \rightarrow x)$
				\State $m = b$, $m2 = c$
				\State Set $d_T(w,r) = a$
			\ElsIf{$b == maximum(a,b,c)$}
				\State $\pi = (x \rightarrow x, y\rightarrow y, z\rightarrow z)$
				\State  $m = a$, $m2 = c$
				\State Set $d_T(w,r) = b$
			\ElsIf{$c== maximum(a,b,c)$}
				\State $\pi = (x \rightarrow y, y \rightarrow x, z \rightarrow z)$
				\State $m = a$, $m2 = b$
				\State Set $d_T(w,r) = c$
			\EndIf
			\If{$d(w,\pi x) == m$ or $d(w,\pi x) == m2$}
				\State push($w$, $Z1[\pi x]$)
			\Else
				\State push($w$, $Z2[\pi x]$)
			\EndIf
		\EndFor
	
	// recurse on each of the zones
	\State $T$ = \textsc{zone1\_recursion}($T,d_T,d,Z1[r],r$)
	\State $T$ = \textsc{zone1\_recursion}($T,d_T,d,Z1[x],x$)
	\State $T$ = \textsc{zone1\_recursion}($T,d_T,d,Z1[y],y$)
	\State $T$ = \textsc{zone1\_recursion}($T,d_T,d,Z1[y],z$)
	
	\State $T$ = \textsc{zone2\_recursion}($T,d_T,d,Z2[x],x,r$)
	\State $T$ = \textsc{zone2\_recursion}($T,d_T,d,Z2[y],y,r$)
	\State $T$ = \textsc{zone2\_recursion}($T,d_T,d,Z2[z],z,r$)
	
	\Return $T$
	
	\EndFunction
	\end{algorithmic}
\end{algorithm}

We can see examples of what happens when we set the new Steiner node for the two different kinds of recursion in Figure \ref{fig:z12universal}

\begin{figure}[!htbp]
\centering
\hfill\includegraphics[width=0.35\linewidth]{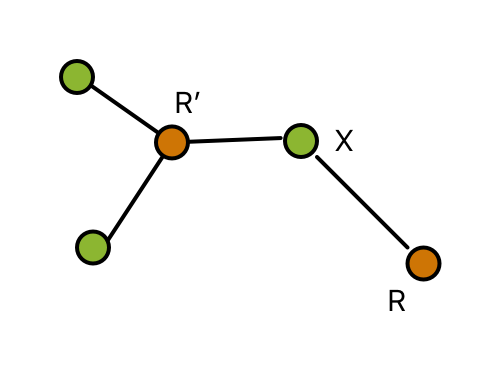}\hfill
\includegraphics[width=0.35\linewidth]{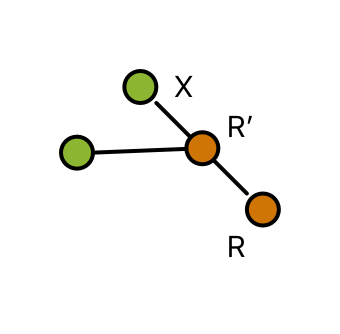}\hfill
\caption{Figure showing the placement of the Steiner node $R'$ for the Zone 1 and Zone 2 recursion. The nodes in orange are Steiner nodes and the nodes in green come from the data set $V$.}
\label{fig:z12universal}
\end{figure}

\end{document}